\documentclass{article}

\usepackage{iclr2024_conference, times}

\usepackage{graphicx}
\usepackage{hyperref}       % hyperlinks
\usepackage{url}            % simple URL typesetting
\usepackage{booktabs}       % professional-quality tables
\usepackage{amsfonts}       % blackboard math symbols
\usepackage{nicefrac}       % compact symbols for 1/2, etc.
\usepackage{microtype}      % microtypography
\usepackage{mathrsfs}
\usepackage{float}
\usepackage{algorithm}
\usepackage{algpseudocode}
\usepackage{comment}
\usepackage{subcaption}
\usepackage{natbib}
\usepackage{enumitem}
  \setlist{leftmargin=*}
\usepackage{lineno}
\usepackage{YK}
\usepackage{wrapfig}
\def\base{\mathrm{base}}
\def\Ex{\bE}

\def\PR{\mathrm{PR}}

\def\reals{\bR}

\title{Learning in reverse causal strategic environments with ramifications on two sided markets}

\author{%
  Seamus Somerstep   \\
  Department of Statistics \\
  University of Michigan\\
  \texttt{smrstep@umich.edu} \\
  \And
 Yuekai Sun \\
  Department of Statistics \\
  University of Michigan \\
  \texttt{yuekai@umich.edu} \\
  \And
   Ya'acov Ritov \\
  Department of Statistics \\
  University of Michigan \\
   \texttt{yritov@umich.edu} \\
}

\iclrfinalcopy
\begin{document}

\lhead{Published as a conference paper at ICLR 2024}
\maketitle

\begin{abstract}
Motivated by equilibrium models of labor markets, we develop a formulation of causal strategic classification in which strategic agents can directly manipulate their outcomes. As an application, we compare employers that anticipate the strategic response of a labor force with employers that do not. We show through a combination of theory and experiment that employers with performatively optimal hiring policies improve employer reward, labor force skill level, and in some cases labor force equity. On the other hand, we demonstrate that performative employers harm labor force utility and fail to prevent discrimination in other cases.
\end{abstract}
\section{Introduction}
%Motivated by equilibrium models of labor markets, we develop a formulation of causal strategic classification in which strategic agents can directly manipulate their outcomes. As an application, we consider employers that seek to anticipate the strategic response of a labor force when developing a hiring policy. We show through a combination of theory and experiment that employers with performatively optimal hiring policies improve employer reward, labor force skill level, and labor force equity (compared to employers that do not anticipate the strategic labor force response) in the classic Coate-Loury labor market model. Empirically, we show that these desirable properties of performative hiring policies do generalize to our own formulation of a general equilibrium labor market. On the other hand, we also observe that the benefits of performatively optimal hiring policies are brittle in some aspects. We demonstrate that in our formulation a performative employer both harms workers by reducing their aggregate welfare and fails to prevent discrimination when more sophisticated wage and cost structures are introduced.
In many applications of predictive modeling, the model itself may affect the distribution of samples on which it has to make predictions; this problem is known as strategic classification \citep{hardt2015Strategic, JMLR:v13:brueckner12a} or performative prediction \citep{perdomo2020Performative}. For example, traffic predictions affect route decisions, which ultimately impact traffic. Such situations can arise in a variety of applications; a common theme is that the samples correspond to strategic agents with an incentive to ``game the system'' and elicit a desired outcome from the model.

In the standard strategic classification setup, the agents are allowed to modify their features, but they do not modify the outcome that the predictive model targets. An example of this is spam classification: spammers craft their messages (\eg\ avoiding certain tokens) to sneak them past spam filters. There is a line of work on \emph{causal strategic classification} that seeks to generalize this setup by allowing the agents to change both their features and outcomes, usually by incorporating a causal model between the two \citep{miller2020Strategic,KleinbergAcm2020,haghtalab2023ondemand,a-tale-of-two-shifts}.

In this paper, motivated by equilibrium models of labor markets (see \citet{FANG2011133S} for a survey), we study a strategic classification setup in which the agents are able to manipulate their attributes via a \emph{reverse causal} mechanism. This complements prior work on causal strategic classification, in which the agents manipulate their attributes via causal mechanisms. As an example of reverse causal strategic classification, we consider the employer's problem in labor market models. In particular, we study the consequences (in terms of employer and labor force welfare) of hiring policies that anticipate reverse causal strategic responses (we will refer to such anticipatory policies as strategic and later performative or optimal). 
\begin{enumerate}
\item In the simple \citet{coate1993Will} labor market model, we show theoretically that such strategic policies lead to higher employer rewards (compared to non-strategic hiring policies). Thus, rational employers should be performative. Further, such hiring policies improve labor force skill level and equity, so performative employers also benefit the labor force.
\item Unfortunately, we also observe that in some aspects, the desirable properties of (reverse causal) strategic hiring policies are brittle. To study their robustness, we developed a more sophisticated general equilibrium labor market model. We show empirically that while our theory generalizes, strategic hiring policies will harm workers by reducing their aggregate welfare and can still lead to disparities amongst labor force participants.
\end{enumerate}
\textbf{Related Work}

\textit{Performative Prediction:}
Performative prediction, introduced in \citet{perdomo2020Performative}, seeks to study distribution shifts that are dependent on model deployment. In particular, the reverse causal distribution map considered in this paper leads to a form of subpopulation shift \cite{maity2021linear,maity2022Understanding}. This line of research has been extended to the stochastic optimization setting in the works \citep{mendler-dunner2020Stochastic, Drusvyatskiy2021, Wood2021,maity2022Predictorcorrector}. Stateful performative prediction, introduced in \citet{statefulpp-2022-aistats}, allows the map $\mathcal{D}$ to depend on both $\theta$ and the current data distribution. The authors of \citet{izzo2021How} and \citet{izzo2022How} propose methods for minimizing the performative risk under a parametric assumption ($\cD(\theta) = p(z;f(\theta))$). The work \citet{outside-the-echo-chamber} establishes the necessary conditions for the performative risk to be convex. The authors of \citet{jagadeesanregret2022} introduce a zero'th order algorithm for minimizing regret in a performative setting. Our algorithm for anti-causal strategic learning is similar in spirit to the work of \citet{izzo2021How}. The main difference is that in our case the performative map is known apriori, which simplifies the procedure and allows for use beyond the parametric setting.

\textit{Strategic Classification:} Although it predates performative prediction, strategic classification \citet{hardt2015Strategic} can be viewed as an instance of performative prediction in which users game their features. The authors of \citet{strat-class-mp-2021, Levanon-Generalized-2022} give efficient algorithms for learning in general strategic settings. The work \citet{Chen-Strategy-Aware-2020} introduces methods for minimizing Stackelberg regret in online strategic classification. The authors of \cite{Yu-Yang-Fan-2022} establish algorithms for a strategic offline reinforcement learning problem. One line of work seeks to study the case where users and the learner do not share all information; the authors of \citet{alt-micro-found-icml-2021} assume users view a noisy version of the model; the work \citet{Dong-Revealed-Preferences-2017} assumes the learner is blind to the strategic agents utility function; the works \citet{pmlr-v139-ghalme21a, BechavodInformation2022, Barsotti-ICJAI-2022} study the case where the strategic agents must also learn the deployed classifier. The authors of \cite{zrnic2021who} study an extension of strategic classification where the roles are reversed; the agents make strategic decisions before the learner deploys a model. In order to introduce interaction among the strategic agents, the authors of \cite{liu22strategic} study a problem where agents compete in contests.  

A line of work similar in vein to our work and strategic classification is causal recourse \cite{algorithmic-recourse-2021}, \cite{improvement-recourse-2023} which seeks to provide actionable intervents for strategic agents to improve their features.

Our work is most aligned with recent efforts to inject causality into strategic classification. Much of the focus has been on improvement (incentivizing agents to improve features in a way that causes responses to improve), beginning with the works \citet{Alon-MultiAgent-2020, KleinbergAcm2020, haghtalab2023ondemand}. The authors of \citet{miller2020Strategic} show that learning good models for improvement is equivalent to solving causal inference problems. A follow-up line of work considers this problem in specific scenarios; with \citet{shravit-caus-strat-LR-2020} focused on regression, and \cite{pmlr-v162-harris22a} on instrumental variables. The work \citet{Dunner-Ding-Wang2022} reframes performative predictions as causal interventions. A closely related work is \citet{a-tale-of-two-shifts}, which gives a method for minimizing the causal strategic empirical risk.

\textit{Economic models of labor markets:} Studies of statistical theories of discrimination began with the two lines of work \citep{arrow1971, phelps1972}. In \citet{phelps1972}, discrimination in labor markets is due to two groups of workers being exogenous, while \citet{arrow1971} shows that even if groups are endogenous, discrimination can occur in equilibrium. These ideas were developed by Coate and Loury into a labor market model in \citep{coate1993Will}. More recently, the authors of \citet{moro2004general} and \citet{moro2003Affirmative} extended this model by including interactions among groups of workers and a wage structure set by inter-employer competition. The line of work \citet{FryarColorBlindAA, FryarWinnerTakeAll, FryarValuingDiversity, NBERw23811} studies the impacts of different action policies (\ie{} color blinded vs color sighted) in a variety of market settings including those similar in spirit to \citep{coate1993Will}. The more contemporary works \citet{Liu2023DisparateEquilibria,KannanDownStream2019} utilize the Coate and Loury model to study discrimination in algorithmic decision making. For a survey on statistical discrimination in labor market, see \citep{FANG2011133S}.

\textbf{Preliminaries}

Performative prediction, introduced in \citet{perdomo2020Performative}, seeks to study distribution shifts that are dependent on model deployment. To be more specific, if the learner picks model parameter $\theta$, then the next samples are drawn from $\mathcal{D}(\theta)$. Thus, a performative learner seeks to minimize \[ \PR(\theta) \triangleq \Ex_{Z\sim \mathcal{D}(\theta)}[\ell(Z;\theta)].\] This problem can be non-convex; most works assume the learner utilizes a strategy of repeatedly retraining a model. Under this strategy, at round $t+1$ of training, the learner deploys
\[\theta_{t+1} = \argmin_{\theta'} \Ex_{Z\sim \mathcal{D}(\theta_t)}[\ell(Z;\theta')].\]
The authors of \citet{perdomo2020Performative} show that this strategy (known as repeated risk minimization) will converge to stable points, which are defined by:
\[\theta_{\text{stab}} \in \argmin_{\theta'} \Ex_{Z\sim \mathcal{D}(\theta_{\text{stab}})}[\ell(Z;\theta')].\]

The second performative solution concept is a performatively optimal point, which satisfies
\[\theta_{\text{opt}} \in \argmin_{\theta'} \Ex_{Z \sim \mathcal{D}(\theta')}[\ell(Z; \theta')].\]
It is worth emphasizing that optimal points are generally not stable, and stable points are generally not optimal; in particular, stable points are not minima of the performative loss, and optimal points need not be the best response to their own induced distribution. 

The canonical example of performative prediction is strategic classification.
%\begin{example}[\citet{hardt2015Strategic}]
%In strategic classification, a learner deploys a predictive model $f_{\theta}(\cdot): \cX \rightarrow \cY$, and in response the samples (strategic agents) game their features at a cost $c(\cdot, \cdot): \cX \times \cX \rightarrow \mathbb{R}^+$:
%\[x \rightarrow x_f \triangleq \argmax_{x'\in\cX} f_\theta(x') - c(x,x').\]
%\end{example}

\begin{example}[\citet{hardt2015Strategic}] Although it predates performative prediction, strategic classification \citet{hardt2015Strategic} can be viewed as an instance of performative prediction in which users game their features. More concretely, there are users with features $x \in \cX$, discrete outcomes $y \in \mathcal{Y} \simeq [K]$, and a corresponding base distribution $P$ over $\mathcal{Z} = \mathcal{X} \times \mathcal{Y}$. The goal of learning is to deploy a model $f: \mathcal{X} \rightarrow \mathcal{Y}$ using training data drawn from $P$. The catch is the following: post-training and pre-testing of $f$ users will game their features (at some cost $c: \mathcal{X} \times \mathcal{X} \rightarrow \mathbb{R^+}$ ) via the update rule:
\[x \rightarrow x_f \triangleq \argmax_{x'\in\cX} f_\theta(x') - c(x,x').\] The goal is to deploy a classifier that is accurate post-distribution shift. The ideal classifier minimizes $\mathbb{E}_{P}\ell(f(x_f), y)$.
\end{example}
\section{Reverse causal strategic learning}
In the standard strategic classification setting, agents update their features $x$, but they do not update their labels $y$. This is common when the agents wish to ``game" the model; the standard example is spam classification: spam creators change their emails to try and avoid spam filters (for example, by avoiding certain tokens), but their emails will remain spam. 

Recently, \citet{a-tale-of-two-shifts} considered a more general problem setting by permitting agents to modify their labels (in addition to their features) through a causal/structural model. Here, the agents still change their features, but changes to the features can propagate through the causal model and (indirectly) lead to changes in their labels. In the following, we consider a \emph{reverse causal} setting in which the agents change their labels and the changes propagate through a structural model to their features. Here are two motivating examples for the reverse causal strategic learning setting: one from labor economics and one from game models of content platforms.

\begin{example}[\cite{coate1993Will}]
\label{ex:coate-loury}
Consider an employer that wishes to hire skilled workers. The worker skill level is represented as $Y\in\{0,1\}$ with $Y \sim \text{Ber}(\pi)$. The employer implements a (noisy) skill level assessment, with the outcome represented as $X\in[0,1]$. The employer receives utility $p_+$ from a ``qualified" hire and suffers loss $p_-$ from an "unqualified" hire, and seeks to train a classifier $f(x):[0,1]\to\{0,1\}$ that optimizes their overall utility:
\[\textstyle
\max_{f\in\cF} \big[p_+ \pi \mathbb{P}(f(x) = 1\mid y=1) - p_- (1-\pi) \mathbb{P}(f(x) = 1 \mid y = 0 )],
\]
Thus far, this is a standard classification problem because the workers are non-strategic, \ie{} they are unaffected by the employer's hiring policy. To introduce a strategic component, the workers are allowed to become qualified (at a cost) in response to the employer's policy. Let $w > 0$ be the wage paid to hired workers and $c$ be the (random and drawn from CDF $G$) cost to a worker of becoming skilled. For an unskilled worker, the expected utility of becoming skilled is
\[\textstyle
u_w(f,y)\triangleq\begin{cases}\int_{[0,1]} wf(x)d\Phi(x\mid 1) - c & \text{if the worker becomes skilled},\\\int_{[0,1]} wf(x)d\Phi(x\mid 0) & \text{if the worker remains unskilled},\end{cases}
\]
where $\Phi(x\mid y)$ is the conditional distribution of skill level assessments. So a strategic worker becomes skilled if
\[\textstyle
w\int_{[0,1]} 1_{\{f(x) = 1\}}d\Phi(x|Y=1) - w \int_{[0,1]} 1_{\{f(x) = 1\}}d\Phi(x|Y=0) -c >0.
\]
\end{example}

\begin{example}[\cite{hron2023modeling}, \cite{Jagadeesan2023modeling}]
\label{ex:exposure-games}
Consider a social media system. There is a demand distribution $u \sim P_d; u \in \mathbb{R}^d$ from which users arrive to the system sequentially, with $u^t$ representing a vector of characteristics of the $t'th$ user. There is also a population of $n$ \textit{content producers} which each produce content $\{s^t_i\}_{i=1}^n; \: s^t_i \in \mathbb{R}^d$. Associated with a piece of content $s$ is some noisy measurement $X \in \mathbb{R}^d$ of the content generated via CDF $\Phi(\cdot \mid s)$. For example, $X$ could consist of the sentiment of ``comment'' $s$ receives from users, as well as the amount of likes and views $s$ generates. In general, we interpret $\mathbb{E} ||X||_2^2$ as the ``amount'' of the content that a creator with signal $X$ produces. The learner wishes to train a recommendation system $R(\Sigma, x, u)$ that recommends content $s_i^t$ to the user $u_t$ with probability: 
\[ P(u_t \text{ is assigned content } s_i^t) = R(\Sigma, x_i^{t}, u^t) = \frac{\exp(\frac{1}{\tau} (x_i^{t})^T\Sigma u^t)}{\sum_{j=1}^n\exp(\frac{1}{\tau} (x_j^t)^T \Sigma u^t)}\]
The learner wishes to maximize the true enjoyment a user receives from content, which is given by
\[r(\Sigma^*, u^t, s^t) = (s^t)^T \Sigma^* u^t.\]
Content producers are not static, however, and will adapt content production. Since exposure to users is directly connected to the learned recommendation system, a content creator with original content $s$ will tend towards producing content $s'$ that optimizes
\[\mathbb{E}_{u \sim P_d; x\sim \Phi(\cdot | s')} R(\Sigma, x, u) - \mathbb{E}_{x\sim \Phi(\cdot \mid s')}||x||_2^2\]
%\YK{elaborate on this example} \jr{Content $s$ not $s'$?}
\end{example}
\subsection{The reverse causal strategic learning problem}
In reverse causal strategic learning, the samples $(X,Y)$ are agents, and the learner wishes to learn a (possibly randomized) policy/rule $f:\cX\to \cY$ to predict the agents' responses $Y$ from their features $X$. The agents are fully aware of the learners policy $f$ and as such, we assume that the response of the agents to $f$ is to change their outcomes $Y$ strategically:
\begin{equation}
\label{eq: reverse causal strategic response}
Y\to Y_+(f,Y)\triangleq\argmax_{y'}W(f,y') - c(y',Y),
\end{equation}
where $W$ is a welfare function that measures the agent's welfare and $c$ is a (possibly random) cost function that encodes the cost (to the agent) of changing their outcome from $Y$ to $y'$.
\begin{example}[Example \ref{ex:coate-loury} cont]
In the labor market example, the agents are the workers, and their welfare is their expected wage
\[\textstyle
W(f,y') = \int_{[0,1]}wf(x)d\Phi(x\mid y'); y'\in \{0,1\},
\]
while the cost is
\[\textstyle
c(y',y) = cy'; \: c \sim G, \: y'\in \{0,1\}.
\]
\end{example}

\begin{example}[Example \ref{ex:exposure-games} cont]
In the content creation example, the agents are the content producers. Their welfare function is
\[\textstyle W(\Sigma, s') = \mathbb{E}_{u \sim P_d; x\sim \Phi(\cdot | s')} R(\Sigma, x, u),\]
while the cost function is simply the effort $\mathbb{E}_{x\sim \Phi(\cdot \mid s')}||x||_2^2$ required to produce a certain ``amount" of content.
\end{example}

The key ingredient that makes this setting \textit{reverse causal} is that strategic change in the outcome propagates to cause a change in the generation of features $X$ via:
\[
X \sim \Phi(\cdot \mid Y ) \to X_+(f,Y) \sim \Phi(\cdot\mid Y_+(f,Y)),
\]
where $\Phi(\cdot\mid Y=y)$ is the conditional distribution of the features $X$ given the outcome $Y$ in the base distribution of the agents. In other words, the agents are unable to directly change their features in the reverse causal setting; \emph{they can only change their features by changing their outcomes.} This distinguishes the reverse causal strategic learning setting from the problem settings in other works on performative prediction and strategic classification \citep{hardt2015Strategic,a-tale-of-two-shifts}.

Going back to the learner, their welfare depends on the post-response agents, so they must account for the strategic response of the agents. A learner which does this is \textit{performative} and minimizes the post-strategic response loss: 
\begin{equation}
\label{eq: anti-causal-objective}
\text{min}_{f \in \cF}\Ex\big[\ell(f(X_+(f,Y)),Y_+(f,Y))\big].
\end{equation}
In the rest of the paper, we will return to the labor market, demonstrating that a learner with a proactive strategy is often mutually beneficial for both the learner and strategic agents.  A discussion on the minimization of Objective \ref{eq: anti-causal-objective} is reserved for appendix A in which we provide the a simple algorithm for the learner to minimize the performative loss in a stochastic setting.
%Going back to the learner, their welfare depends on the post-response agents, so they must account for the strategic response of the agents. The first option the learner has is to be reactive (non-performative) and repeatedly re-train a model, each time selecting a model that is best suited to the distribution induced by the prior model:
%\begin{equation}
%\label{eq: reactive-objective}
%f_{t+1} = \text{min}_{f \in \cF}\Ex\big[\ell(f(X_+(f_t,Y)),Y_+%(f_t,Y))\big].
%\end{equation}
%This strategy is directly analogous to repeated risk minimization in \cite{perdomo2020Performative}.
%The second is to be strategic (performative) and minimize the post-strategic response loss (also known as the performative loss):
%\begin{equation}
%\label{eq: anti-causal-objective}
%\text{min}_{f \in \cF}\Ex\big[\ell(f(X_+(f,Y)),Y_+(f,Y))\big],
%\end{equation}
%In the rest of the paper, we will return to the labor market, demonstrating that a learner with a proactive strategy is often mutually beneficial for both the learner and strategic agents.  A discussion on the minimization of Objective \ref{eq: anti-causal-objective} is reserved for appendix A in which we provide the a simple algorithm for the learner to minimize the performative loss in a stochastic setting.

\section{Strategic hiring in the Coate-Loury model}
In this section, we focus on example \ref{ex:coate-loury} and study the impact of performatively optimal hiring policies on employer and labor force welfare in the Coate-Loury model. Despite the strategic nature of the labor force, prior studies of labor market dynamics generally assume employers are merely reactive (instead of proactive) to the strategic responses of the labor force \citep{FANG2011133S}.
We show that in a variety of market settings, performatively optimal hiring policies improve employer welfare, labor force welfare, and labor market equity.
\subsection{Performative Prediction in the Coate-Loury model}
Recall the setup from example \ref{ex:coate-loury}. We impose some standard assumptions on the problem:
\begin{enumerate}
\item $\frac{P(X=x\mid y=1)}{P(X =x \mid y=0)} = \frac{\phi(x \mid y =1)}{\phi(x \mid y=0)}$ is monotonically increasing in $x$,
\item $\phi(x\mid y)$ is continuously differentiable for $y\in\{0,1\}$,
\item  $G(0) = 0$, $G$ is continuously differentiable, and $c\sim G$ is almost surely bounded above by $M_G$.
\end{enumerate}
The second two assumptions are technical, but the first is more substantive. It resembles the monotone likelihood ratio assumption in statistical hypothesis testing that ensures consistency of hypothesis tests. In this case, it ensures the optimal hiring policy is a threshold policy of the form $\ones\{x\ge\theta\}$ for some $\theta$, this is the operating assumption in the original \cite{coate1993Will} paper, and we will proceed with it as well.

Armed with the assumption of a threshold hiring policy, we discuss performative prediction in the \citet{coate1993Will} model. Recall the strategic response of the labor force participants in example \ref{ex:coate-loury}. In aggregate, the proportion of skilled labor force participants becomes
\[ \pi(\theta) = G(w[\mathbb{P}(x>\theta\mid y=1) - \mathbb{P}(x>\theta\mid y=0)]).\]
Given this, the employer's strategic hiring problem is really an instance of performative prediction with the performative employer utility given by
\[U_{\text{perf}}(\theta) \triangleq p_+\mathbb{P}(X>\theta \mid y=1) \pi(\theta) - p_-\mathbb{P}(X>\theta \mid y=0) (1-\pi(\theta)).\]
As in other performative settings, the employer may opt to deploy an optimal or stable policy. An employer that deploys an optimal policy anticipates and accounts for the labor force participants responses to their decisions and thus deploys a policy that solves the performative problem
\begin{equation}
\theta_\text{opt} \in \argmax_\theta \:{} U_{\text{perf}}(\theta). 
\label{eq:strategic-employer-problem}
\end{equation}
An employer that deploys a stable policy is reactive rather than anticipatory towards labor force strategic responses, \ie\:they deploy a stable policy, which is any policy that maximizes employer utility on its own induced distribution:
\[\theta_{\text{stab}} \in \text{argmax}_{\theta} \:  p_+ \pi(\theta_\text{stab}) \mathbb{P}(X>\theta \mid y=1) - p_- (1-\pi(\theta_{\text{stab}})) \mathbb{P}(X>\theta \mid y = 0 ).\]

The concept of performative vs stable solutions in the context of a micro-economic model has a game theoretic interpretation. A performative solution is the Stackleberg equilibrium to a 2-player game between the firm and the labor force, with the firm as the leader and the labor force as the follower. On the other hand, the stable solution is the Nash equilibrium between the firm and the labor force if there is no leader-follower dynamic.

We will study both low-wage and high-wage markets. Low-wage markets possess the desirable property that any employer that follows a ``greedy" strategy, \ie\:they sequentially deploy  
\begin{equation}
\theta_{t+1} \gets \argmax_\theta p_+\mathbb{P}(X>\theta \mid y=1) \pi(\theta_t) - p_-\mathbb{P}(X>\theta \mid y=0) (1-\pi(\theta_t)),
\label{eq:non-strategic-employer-problem}
\end{equation}
will eventually converge on a stable policy. This is because \eqref{eq:non-strategic-employer-problem} is an instance of RRM, and low worker wages (in addition to some regularity conditions) will ensure that the requirements for RRM convergence provided in \citet{perdomo2020Performative} are satisfied (see Appendix D).

High-wage markets are independently interesting, as in such markets the benefits of optimal policies for a firm will be substantial. On the other hand, in high-wage markets, the convergence of RRM may not be universal. This is not necessary for any of our results (stable policies will always \textit{exist} and any greedy employer that does stabilize will do so on a stable policy). Additionally, our empirical results also demonstrate that, in practice, convergence of reactive firms is not sensitive to market conditions.

Finally, we remark that each theorem will need additional assumptions on the structure of the labor market (for example, worker wage and firm reward). The social situations that the \citet{coate1993Will} model applies to are broad \citet{fang2011Chapter}; and thus the applicability of each theorem will depend on the social environment at hand. For any specific social situation, a full justification of these assumptions would require a study of the model on real labor market data (or the alternative social situation). Such studies are relatively rare in the literature, but two examples are \citet{ArcidiaconoAA2011} (studies admissions into Duke), \citet{AltonjiStatDisc2001} (studies education data).
\subsection{Effects of strategic hiring on employer and labor force welfare}
The main result of this subsection compares employer welfare (in terms of the employer's expected utility) and labor force welfare (in terms of the fraction of skilled labor force participants) resulting from optimal and stable hiring policies.
We impose two additional assumptions on the market. 
\begin{enumerate}
\item there exists $\tilde{\theta} \in [0,1]$ such that $\mathbb{P}(X>\tilde{\theta}\mid y=1) - \mathbb{P}(X>\tilde{\theta}\mid y=0) > \delta_1 > 0$,
\item $\frac{\phi(x|1)}{\phi(x|0)} > \delta_2 > 0$ for any $x\in[0,1]$.
\end{enumerate}
Together, these ensure that some distinction between skilled and unskilled workers is possible, but no policy can perfectly distinguish between the two. These assumptions ensure that stable and optimal policies will be distinct.

We study markets in both the low-wage and high-wage regimes. In the large-wage regime, firms receive substantial benefits from being optimal.
\begin{theorem}
\label{thm:welfare-effects}
If $w > M_G/\delta_1$ and $p_+ > p_-/(\delta_1 \delta_2)$ then the following holds for all stable parameters $\theta_{\text{stab}}$ and all optimal parameters $\theta_{\text{opt}}$:
\begin{enumerate}
    \item $\pi(\theta_{\text{opt}}) > \pi(\theta_{\text{stab}})$.
    \item $U_{\text{perf}}(\theta_\text{stab}) \leq U_{\text{perf}}(\theta_{\text{opt}})/(1+\delta_1)$. 
\end{enumerate}
\end{theorem}
On the other hand, in the low-wage regime, the firm's benefits from an optimal strategy may be small.
\begin{theorem}
\label{thm:welfare-effects-2}
If $w > 0$, $p_+ > \text{max}(1, \frac{ \pi(\tilde{\theta})}{(1-\pi(\tilde{\theta}))}p_-)$ and $\delta_1 \delta_2 > p_-$, then the following holds for all stable parameters $\theta_{\text{stab}}$ and all optimal parameters $\theta_{\text{opt}}$:
\begin{enumerate}
    \item $\pi(\theta_{\text{opt}}) > \pi(\theta_{\text{stab}})$.
    \item $U_{\text{perf}}(\theta_\text{stab}) \leq U_{\text{perf}}(\theta_{\text{opt}})$. 
\end{enumerate}
\end{theorem}
To see the intuition behind Theorems \ref{thm:welfare-effects} and \ref{thm:welfare-effects-2}, consider an extreme case where $p_- = 0$. Here, since for any $\pi>0$, the best reactive employer response is $\theta = 0$, any stable policy will satisfy $\pi(\theta_{\text{stab}}) = 0$ (and thus $U_{\text{perf}}(\theta_{\text{stab}}) = 0$), while clearly an optimal employer can do better than this and will in general deploy a model that results in $\pi(\theta_{\text{opt}})>0$. Theorems \ref{thm:welfare-effects} and \ref{thm:welfare-effects-2} are a generalization of these simple dynamics.

The alignment between the social welfare of the workers and the reward to the learner is in contrast with other areas of strategic learning. For example, \citet{milli2018Social} show that there is a direct trade-off between the social burden on strategic agents and learner accuracy in strategic classification.

\subsection{Effects of strategic hiring on labor force equity}

Besides labor force skill level, another pressing issue in labor markets is equity. In fact, the Coate-Loury model was developed to show how inequities may arise in labor markets despite the lack of explicit discrimination in the market. First, we recall the original two-group version of the Coate-Loury model.

Employers now seek to hire workers from a large population of labor force participants consisting of two identifiable groups denoted as $\text{Maj}$ or $\text{Min}$, with $\lambda$ denoting the fraction of workers with a $\text{Maj}$ group membership. Crucially, employer profits, worker costs, wages, and worker signals are agnostic with respect to group membership. Denoting the proportion of qualified workers in a group as $\pi^{\text{Maj}}$ and $\pi^{\text{Min}}$ the (non-performative) employer utility is:
\[U(\theta^{\text{Maj}}, \theta^{\text{Min}}) =  \lambda U(\theta^{\text{Maj}}) + (1-\lambda) U(\theta^{\text{Min}}),\] 
%\begin{multline}
%\begin{gathered}
%U(\theta^{\text{Maj}}, \theta^{\text{Min}}) =  \lambda U(\theta^{\text{Maj}}) + (1-\lambda) U(\theta^{\text{Min}}), \\
%U(\theta^{\text{Maj}}) = p_+ \pi^{\text{Maj}}P(x>\theta^{\text{Maj}}\mid Y=1) - p_- (1-\pi^{\text{Maj}}) P(x>\theta^{\text{Maj}}\mid Y=0),\\
%U(\theta^{\text{Min}}) = p_+ \pi^{\text{Min}}P(x>\theta^{\text{Min}}\mid Y=1) - p_- (1-\pi^{\text{Min}}) P(x>\theta^{\text{Min}}\mid Y=0).\\
%\end{gathered}
%\end{multline}
The population level response of labor force participants to deployed policies is largely similar in the group case:
\begin{multline}
\begin{gathered}
\pi^{\text{Maj}}(\theta^{\text{Maj}}, \theta^{\text{Min}}) = G(w[P(x>\theta^{\text{Maj}}\mid y=1) - P(x>\theta^{\text{Maj}}\mid y=0)]),\\
\pi^{\text{Min}}(\theta^{\text{Maj}}, \theta^{\text{Min}}) = G(w[P(x>\theta^{\text{Min}}\mid y=1) - P(x>\theta^{\text{Min}}\mid y=0)]).\\
\end{gathered}
\end{multline}
Note that under these assumptions, both the employer's non-performative hiring problem and the performative one are seperable; \ie{} the employer can simply solve two hiring problems for each group separately. As such, we define a stable pair $\vec{\theta}_{\text{stab}} = (\theta_{\text{stab}}^{\text{Maj}}, \theta_{\text{stab}}^{\text{Min}})$ as a pair of policies that are each stable and optimal pairs $\vec{\theta}_{\text{opt}} = (\theta_{\text{opt}}^{\text{Maj}}, \theta_{\text{opt}}^{\text{Min}})$ as a pair of policies that are each (performatively) optimal.

In \cite{coate1993Will} a policy pair $\vec{\theta}$ is discriminatory if $\pi^{\text{Maj}}(\vec{\theta}) \neq \pi^{\text{Min}}(\vec{\theta})$, this will also be our metric for discrimination in strategic hiring. It is worth emphasizing that our definition of a discriminatory stable pair is exactly the definition of a discriminatory market equilibrium in \cite{coate1993Will}.
\begin{comment}
Stable and optimal parameters are defined in a similar manner. A stable parameter is any pair $\vec{\theta}_{\text{stab}} = (\theta_{\text{stab}}^{\text{Maj}}, \theta_{\text{stab}}^{\text{Min}})$ which satisfies:
\begin{multline}
\begin{gathered}
\theta_{\text{stab}}^{\text{Maj}} \in \argmax_{\theta} \pi^{\text{Maj}}(\vec{\theta}_{\text{stab}}) \mathbb{P}(X>\theta \mid y=1) - p_- (1-\pi^{\text{Maj}})(\vec{\theta}_{\text{stab}}) \mathbb{P}(X>\theta \mid y = 0 ),\\
\theta_{\text{stab}}^{\text{Min}} \in \argmax_{\theta} \pi^\text{Min}(\vec{\theta}_{\text{stab}}) \mathbb{P}(X>\theta \mid y=1) - p_- (1-\pi^\text{Min}(\vec{\theta}_{\text{stab}})) \mathbb{P}(X>\theta \mid y = 0 ).\\
\end{gathered}
\end{multline}
While optimal points are pairs $\vec{\theta}_{\text{opt}} = (\theta_{\text{opt}}^{\text{Maj}}, \theta_{\text{opt}}^{\text{Min}})$ which satisfy:
\begin{multline}
\begin{gathered}
\theta_{\text{opt}}^{\text{Maj}} \in \argmax_{\theta} \pi^{\text{Maj}}(\theta) \mathbb{P}(X>\theta \mid y=1) - p_- (1-\pi^{\text{Maj}}(\theta)) \mathbb{P}(X>\theta \mid y = 0 ),\\
\theta_{\text{opt}}^{\text{Min}} \in \argmax_{\theta} \pi^{\text{Min}}(\theta) \mathbb{P}(X>\theta \mid y=1) - p_- (1-\pi^{\text{Min}}(\theta)) \mathbb{P}(X>\theta \mid y = 0 ).\\
\end{gathered}
\end{multline}
\end{comment}

We will see that, in a certain sense, an employer with a performatively optimal strategy enforces fairness amongst the groups, while a reactive employer provides no such guarantees. Some technical assumptions are needed on the conditions of the market; as the employer's hiring problem is separable with respect to group, we state these assumptions in the context of a single group market, with the implication that the requirements hold when the market is constrained to either group.

We will study markets with the following assumptions:
\begin{enumerate}
    \item There exists $\tilde{\theta}$ such that $\mathbb{P}(x>{\tilde{\theta}}|y=1) - \mathbb{P}(x> {\tilde{\theta}}|y=0)>1-\epsilon$
    \item $p_+ = p_-$
\end{enumerate}
The second assumption is primarily to simplify the problem. On the other hand, the first assumption is crucial. It implies that the firm find a hiring policy that provides excellent separation between qualified and unqualified workers (generally $\epsilon$ can be thought of as small). Under such an assumption, a performative firm will have both ability and motive to steer the market towards equitable equilibrium.
\begin{theorem}
\label{thm:equity-effect}
Assume $\theta^{*-1}(\theta)$ is $c-$Lipschitz on $[0, \tilde{\theta}]$, where $\theta^*(\pi)$ is the best employer response if $\mathbb{P}(y=1) = \pi$ and that $w>  M_G/(1-\epsilon)$  Then the following hold simultaneously:
    \begin{enumerate}
        \item There exists a stable pair $\vec{\theta}_{\text{stab}}$ such that $|\pi^{\text{Maj}}(\vec{\theta}_{\text{stab}}) - \pi^{\text{Min}} (\vec{\theta}_{\text{stab}})| > 1-c$.
        \item $|\pi^{\text{Maj}}(\vec{\theta}_{\text{opt}}) -\pi^{\text{Min}}(\vec{\theta}_{\text{opt}})| < \epsilon$ for all optimal pairs $\vec{\theta}_{\text{opt}}.$
    \end{enumerate}
\end{theorem}

The first assumption requires sufficient smoothness on the inverse of the firm's best response; in general, it is easy to construct markets where both $c$ and $\epsilon$ are small. Such an example is presented in Appendix B, and some sufficient market conditions for these conditions to hold are provided in Appendix C.

We also study the problem in low-wage markets, a strong concavity assumption on the firms non-performative utility $U(\theta)$ is needed. This is generally not too difficult to meet, and a discussion is supplied in Appendix D.

\begin{theorem}
\label{thm:equity-effect-2}
Assume that the non-performative firm utility $U(\theta)$ is $\gamma-$ strongly concave for all $\pi$. Also, assume $g(\cdot)$, and $\phi(\cdot())$ are bounded above by $K_1$, and they are differentiable and $g'(\cdot)$, $\phi'(\cdot)$ are bounded above by $K_2$. If $G^{-1}(\frac{\phi(\tilde{\theta}|0)}{(1-\epsilon)(\phi(\tilde{\theta}|1) + \phi(\tilde{\theta}|0))}) < w < \gamma/(2 K_1 K_2)$ then the following hold simultaneously:  
    \begin{enumerate}
        \item There exists a stable pair $\vec{\theta}_{\text{stab}}$ such that $|\pi^{\text{Maj}}(\vec{\theta}_{\text{stab}}) - \pi^{\text{Min}} (\vec{\theta}_{\text{stab}})| > 0$.
        \item $|\pi^{\text{Maj}}(\vec{\theta}_{\text{opt}}) -\pi^{\text{Min}}(\vec{\theta}_{\text{opt}})| = 0$ for all optimal pairs $\vec{\theta}_{\text{opt}}.$
    \end{enumerate}
\end{theorem}
For this to truly be a ``low-wage market" we need $G^{-1}(\frac{\phi(\tilde{\theta}|0)}{(1-\epsilon)(\phi(\tilde{\theta}|1) + \phi(\tilde{\theta}|0))}) \approx 0$, we will see an example of markets that satisfy this and the other assumptions in the appendix.

The intuition behind optimal policies enforcing fairness in Theorems \ref{thm:equity-effect} and \ref{thm:equity-effect-2} is relatively simple; a performative employer that can distinguish skill level well posseses both direct incentive and ability to ensure that both majority and minority group worker qualification levels will be high, and as such the discriminatory gap will be small. On the other hand, a reactive employer has no direct incentive to increase the fraction of skilled workers; \ie\ they are too myopic to maximize the fraction of skilled workers. Thus, there is no invisible hand steering the market to more equitable equilibria. 
\section{Experiments on markets with continuous skill levels}
In this section, we study the sensitivity of the promising labor force skill and equity results in the preceding section to the underlying Coate-Loury model. A key limitation of the two-group Coate-Loury model is that the labor markets for the two groups operate independently of each other. Here we study the problem using a more sophisticated general equilibrium model (introduced in Appendix B) which allows for such cross-group effects. In our formulation of a labor market, we empirically observe the following economic takeaways: 
\begin{enumerate}
    \item As in the Coate-Loury model, the employer is always incentivized to be performative. Also, a performative employer benefits labor force participants because it increases labor force skill levels.
    \item Unfortunately, a performative employer harms workers by reducing their aggregate welfare. Additionally, the fairness benefits of a performative employer are brittle with respect to the assumptions of the underlying market.
\end{enumerate}
Due to space limitations, we only briefly describe our modifications to the Coate-Loury model here and defer a detailed description of the market to Appendix B. The first is that worker skill level $Y$ has Lebesgue density $p(y)$, implying that qualification is not binary, but rather a continuum of possible productivity levels. The second modification is to the reverse causal updates of the workers; while they remain strategic and update their outcomes according to the reverse causal mechanism \ref{eq: reverse causal strategic response}, the wage structure and cost are different. The cost is now a fixed function: $c(\cdot, \cdot): \cY \times \cY \rightarrow \mathbb{R}^+.$ The wage and utility structure will be more complex, again we refer to Appendix B for a detailed description of changes.
\subsection{Experiments With Linear Utility and Flat Wages}
In the \citet{coate1993Will} model, it is assumed that production functions are linear and that the wage structure is flat. In our context, this is the case where $w(x, f) = w$ and $u(y, H_{f}(y)) = H_{\pi}(y)(ay-1)$. Finally, the cost is specified to be quadratic: $c(y', y) = \frac{c}{2}(y'-y)^2$.

\begin{comment}
 \begin{figure}[htb]
  \begin{subfigure}[b]{0.43\textwidth}
    \includegraphics[width=\textwidth]{New Plots/sec5test1U.png}
    %\caption{employer Utility vs Round }
  \end{subfigure}
  \begin{subfigure}[b]{0.43\textwidth}
    \includegraphics[width=\textwidth]{New Plots/sec5test1P.png}
    %\caption{Policy Deployments vs Round }
  \end{subfigure}
    % \vspace{-1.5pc}
  \caption{Algorithm 1 and RSGD with $a = 4, \: w = 1, \: c = 4$}
  \label{Fig: alg1-vs-rsgd}
\end{figure}
\end{comment}
To study worker welfare and employer utility, we assume $p(y)$ is known, so the employer's policy optimization is non-stochastic. Figure \ref{Fig:Lin-Worker-Welfare} diagrams the worker side. In general, an optimal policy results in a larger portion of qualified workers, while a stable policy leads to greater aggregate worker welfare. The right-hand side of Figure \ref{Fig: alg1-vs-RSGD-2} plots the firm side; unsurprisingly, the employer prefers an optimal policy; additionally, the left-hand side of Figure \ref{Fig: alg1-vs-RSGD-2} demonstrates algorithm 1 (appendix A) in a labor market.
\begin{comment}
\begin{figure}[]
  \begin{subfigure}[b]{0.49\textwidth}
    \includegraphics[width=\textwidth]{qual_work_a_c.png}
    \caption{Left: $w=1,c=8$ // Right: $w=1, a=2$ }
  \end{subfigure}
  \hfill
  \begin{subfigure}[b]{0.49\textwidth}
    \includegraphics[width=\textwidth]{work_ut_a_c.png}
    \caption{Left: $w=1,c=8$ // Right: $w=1, a=2$ }
  \end{subfigure}
  \caption{Worker Welfare and Proportion of Qualified Workers for Optimal and Stable Policies}
  \label{Fig:Lin-Worker-Welfare}
\end{figure}
\end{comment}
\begin{figure}
  \begin{subfigure}[b]{0.49\textwidth}
    \includegraphics[width=\textwidth]{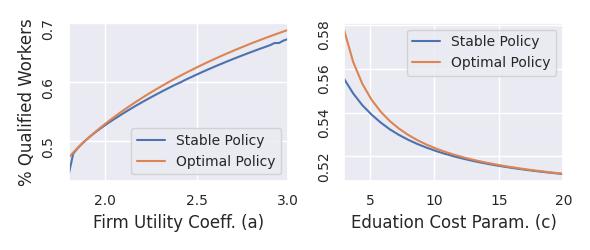}
    \caption{Left: $w=1,c=8$ // Right: $w=1, a=2$ }
  \end{subfigure}
  \hfill
  \begin{subfigure}[b]{0.49\textwidth}
    \includegraphics[width=\textwidth]{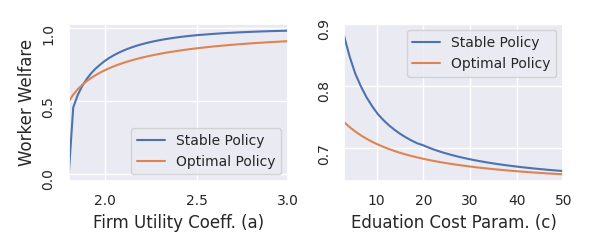}
    \caption{Left: $w=1,c=8$ // Right: $w=1, a=2$ }
  \end{subfigure}
  \caption{Worker Welfare and Proportion of Qualified Workers for Optimal and Stable Policies}
  \label{Fig:Lin-Worker-Welfare}
\end{figure}
\begin{figure}[]
  \begin{subfigure}[b]{0.25\textwidth}
    \includegraphics[width=\textwidth]{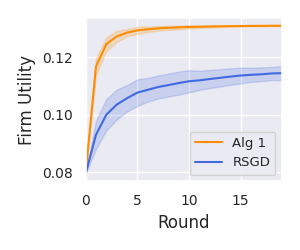}
  \end{subfigure}
  \begin{subfigure}[b]{0.25\textwidth}
    \includegraphics[width=\textwidth]{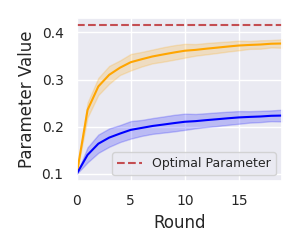}
  \end{subfigure}
  \hfill
  \begin{subfigure}[b]{0.49\textwidth}
    \includegraphics[width=\textwidth]{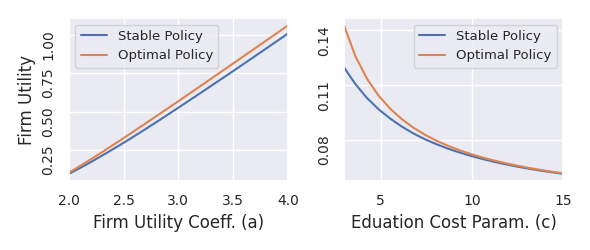}
  \end{subfigure}
  \caption{Left: RSGD and Alg 1 ($a=2, w=1, c=5$) // Right: employer Utility Under Policies (Left: $w=1,c=5$ // Right: $w=1, a=2$)}
  \label{Fig: alg1-vs-RSGD-2}
\end{figure}

\subsection{Experiments with non-linear utility and non-flat wages}
Next, we present experiments on markets that include identifiable groups and move beyond the assumptions of linear utility and flat wages. In the subsequent experiments, wages are non-linear and are determined by a Nash equilibrium among competing employers (inspired by \citep{moro2004general}). Additionally, we assume $u(y, H_{f}(y)) = y (H_{f}(y))^{\alpha}$, the important difference being that $u(\cdot, \cdot)$ is non-linear in hiring probability, which removes the separability property present in the Coate-Loury model. Finally, we no longer assume that cost is agnostic with respect to group; cost is now specified by $c(y', y, i) = \frac{c_i}{2}(y'-y)^2; {} i \in \{\text{Maj}, \text{Min}\}$.

Figure \ref{Fig: Nonlin-Worker-Welfare} diagrams the welfare of each group under different policies and parameters; in the top row, group costs and proportions are adjusted (costs are adjusted so total cost remains constant), while in the bottom row, the utility parameter $\alpha$ is adjusted.

\begin{figure}
  \begin{subfigure}[b]{0.49\textwidth}
    \includegraphics[width=\textwidth]{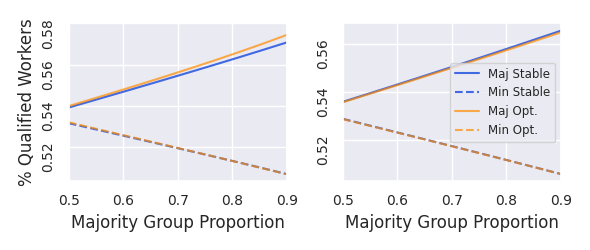}
    \caption{Left: $\alpha=2$ // Right: $\alpha = 0.5$ }
  \end{subfigure}
  \hfill
  \begin{subfigure}[b]{0.49\textwidth}
    \includegraphics[width=\textwidth]{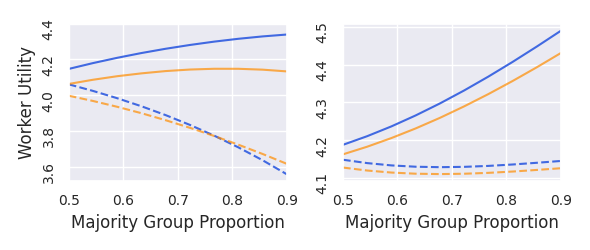}
    \caption{Left: $\alpha=2$ // Right: $\alpha = 0.5$ }
  \end{subfigure}
  \vfill
  \begin{subfigure}[b]{0.49\textwidth}
    \includegraphics[width=\textwidth]{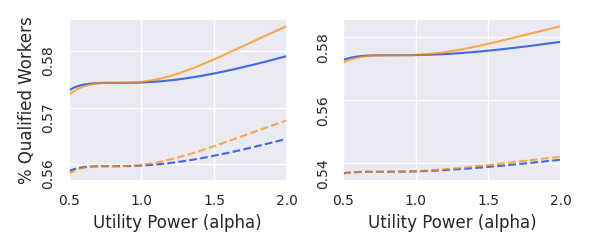}
    \caption{Left: $c_\text{Min} - c_\text{Maj}=5$ // Right: $c_\text{Min} - c_\text{Maj} = 20$ }
  \end{subfigure}
  \hfill
  \begin{subfigure}[b]{0.49\textwidth}
    \includegraphics[width=\textwidth]{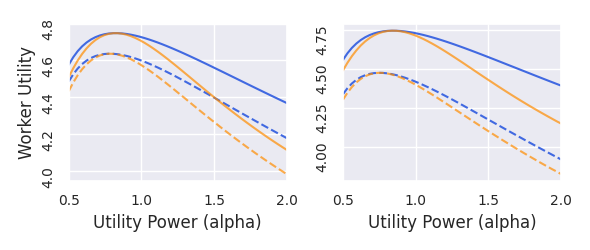}
    \caption{Left: $c_\text{Min} - c_\text{Maj}=5$ // Right: $c_\text{Min} - c_\text{Maj} = 20$ }
  \end{subfigure}
  \caption{Worker Welfare and Proportion of Qualified Workers for Optimal and Stable Policies}
  \label{Fig: Nonlin-Worker-Welfare}
\end{figure}
A similar pattern as before emerges in Figure \ref{Fig: Nonlin-Worker-Welfare}. Workers tend to be more qualified under performative policies (though this is broken in Figure \ref{Fig: Nonlin-Worker-Welfare} when the utility function is concave in the proportion of qualified workers), but they tend to have a higher average welfare under stable policies. Clearly, optimal policies no longer enforce fairness, as there is no appreciable difference in discrimination between the two.

\section{Conclusion}
We have introduced a problem setting that extends strategic classification by allowing agents to directly manipulate $y$, and in turn, cause a shift in $x$; additionally, we provided an algorithm for risk minimization in this setting. As an application of this new framework, we studied the effects of stable and optimal strategies on strategic agent welfare and equity in labor market models. We have demonstrated that a strategic/performative employer can help or harm strategic workers depending on the measurement of agent welfare used. In general, a proactive learner assists the strategic workers in becoming more skilled but harms the strategic agents by reducing their overall aggregate welfare. Additionally, in some but not all cases, a performative learner can assist in preventing discriminatory equilibrium in labor markets. Finally, we have seen that an employer will always benefit from deploying a performative/strategic hiring policy.
\subsubsection*{Acknowledgments}
This paper is based upon work supported by the National Science Foundation (NSF) under grants no. 2027737 and 2113373.
\bibliographystyle{plainnat}
\bibliography{YK, seamus}

\begin{thebibliography}{55}
\providecommand{\natexlab}[1]{#1}
\providecommand{\url}[1]{\texttt{#1}}
\expandafter\ifx\csname urlstyle\endcsname\relax
  \providecommand{\doi}[1]{doi: #1}\else
  \providecommand{\doi}{doi: \begingroup \urlstyle{rm}\Url}\fi

\bibitem[Alon et~al.(2020)Alon, Dobson, Procaccia, Talgam-Cohen, and Tucker-Foltz]{Alon-MultiAgent-2020}
Tal Alon, Magdalen Dobson, Ariel~D. Procaccia, Inbal Talgam-Cohen, and Jamie Tucker-Foltz.
\newblock Multiagent evaluation mechanisms.
\newblock In \emph{AAAI Conference on Artificial Intelligence}, 2020.

\bibitem[Altonji and Pierret(2001)]{AltonjiStatDisc2001}
Joseph~G. Altonji and Charles~R. Pierret.
\newblock Employer learning and statistical discrimination.
\newblock \emph{The Quarterly Journal of Economics}, 116\penalty0 (1):\penalty0 313--350, 2001.

\bibitem[Arcidiacono et~al.(2011)Arcidiacono, Aucejo, Fang, and Spenner]{ArcidiaconoAA2011}
Peter Arcidiacono, Esteban~M. Aucejo, Hanming Fang, and Kenneth~I. Spenner.
\newblock Does affirmative action lead to mismatch? a new test and evidence.
\newblock \emph{Quantitative Economics}, 2\penalty0 (3):\penalty0 303--333, 2011.

\bibitem[Arrow(1971)]{arrow1971}
Kenneth Arrow.
\newblock The theory of discrimination.
\newblock \emph{Labor Economics vol 4}, 1971.

\bibitem[Barsotti et~al.(2022)Barsotti, Kocer, and Santos]{Barsotti-ICJAI-2022}
Flavia Barsotti, R{\"u}ya~G{\"o}khan Kocer, and Fernando~P Santos.
\newblock Transparency, detection and imitation in strategic classification.
\newblock In \emph{International Joint Conferences on Artificial Intelligence}, 2022.

\bibitem[Bechavod et~al.(2022)Bechavod, Podimata, Steven~Wu, and Ziani]{BechavodInformation2022}
Yahav Bechavod, Chara Podimata, Zhiwei Steven~Wu, and Juba Ziani.
\newblock Information discrepancy in strategic learning.
\newblock In \emph{International Conference on Machine Learning}, 2022.

\bibitem[Brown et~al.(2022)Brown, Hod, and Kalemaj]{statefulpp-2022-aistats}
Gavin Brown, Shlomi Hod, and Iden Kalemaj.
\newblock {Performative Prediction in a Stateful World}.
\newblock In \emph{International Conference on Artificial Intelligence and Statistics}, 2022.

\bibitem[Br{{\"u}}ckner et~al.(2012)Br{{\"u}}ckner, Kanzow, and Scheffer]{JMLR:v13:brueckner12a}
Michael Br{{\"u}}ckner, Christian Kanzow, and Tobias Scheffer.
\newblock Static prediction games for adversarial learning problems.
\newblock \emph{Journal of Machine Learning Research}, 13\penalty0 (85):\penalty0 2617--2654, 2012.
\newblock URL \url{http://jmlr.org/papers/v13/brueckner12a.html}.

\bibitem[Chen et~al.(2020)Chen, Liu, and Podimata]{Chen-Strategy-Aware-2020}
Yiling Chen, Yang Liu, and Chara Podimata.
\newblock Learning strategy-aware linear classifiers.
\newblock In \emph{Conference On Neural Information Processing Systems}, 2020.

\bibitem[Coate and Loury(1993)]{coate1993Will}
Stephen Coate and Glenn~C. Loury.
\newblock Will {{Affirmative-Action Policies Eliminate Negative Stereotypes}}?
\newblock \emph{The American Economic Review}, 83\penalty0 (5):\penalty0 1220--1240, 1993.
\newblock ISSN 0002-8282.

\bibitem[Craig and Fryer(2017)]{NBERw23811}
Ashley~C Craig and Jr~Fryer, Roland~G.
\newblock Complementary bias: A model of two-sided statistical discrimination.
\newblock Working Paper 23811, National Bureau of Economic Research, 2017.

\bibitem[Dong et~al.(2018)Dong, Roth, Schutzman, Waggoner, and Wu]{Dong-Revealed-Preferences-2017}
Jishuo Dong, Aaron Roth, Zachary Schutzman, Bo~Waggoner, and Zhiwei~Steven Wu.
\newblock Strategic classification from revealed preferences.
\newblock In \emph{Proceedings of the 2018 ACM Conference On Economics and Computation}, 2018.

\bibitem[Drusvyatskiy and Xiao(2020)]{Drusvyatskiy2021}
Dmitriy Drusvyatskiy and Lin Xiao.
\newblock Stochastic optimization with decision-dependent distributions.
\newblock \emph{https://arxiv.org/abs/2011.11173}, 2020.

\bibitem[Fang and Moro(2011{\natexlab{a}})]{FANG2011133S}
Hanming Fang and Andrea Moro.
\newblock Chapter 5 - theories of statistical discrimination and affirmative action: A survey.
\newblock volume~1 of \emph{Handbook of Social Economics}, pages 133--200. North-Holland, 2011{\natexlab{a}}.
\newblock \doi{https://doi.org/10.1016/B978-0-444-53187-2.00005-X}.
\newblock URL \url{https://www.sciencedirect.com/science/article/pii/B978044453187200005X}.

\bibitem[Fang and Moro(2011{\natexlab{b}})]{fang2011Chapter}
Hanming Fang and Andrea Moro.
\newblock Chapter 5 - {{Theories}} of {{Statistical Discrimination}} and {{Affirmative Action}}: {{A Survey}}.
\newblock In Jess Benhabib, Alberto Bisin, and Matthew~O. Jackson, editors, \emph{Handbook of {{Social Economics}}}, volume~1, pages 133--200. {North-Holland}, January 2011{\natexlab{b}}.
\newblock \doi{10.1016/B978-0-444-53187-2.00005-X}.

\bibitem[Fryar and Loury(2005)]{FryarWinnerTakeAll}
Roland~G. Fryar and Glenn~C. Loury.
\newblock Affirmative action in winner-take-all markets.
\newblock \emph{The Journal of Economic Inequality}, 2005.

\bibitem[Fryar and Loury(2013)]{FryarValuingDiversity}
Roland~G. Fryar and Glenn~C. Loury.
\newblock Valuing diversity.
\newblock \emph{Journal of Political Economy}, 2013.

\bibitem[Fryar et~al.(2008)Fryar, Loury, and Yuret]{FryarColorBlindAA}
Roland~G. Fryar, Glenn~C. Loury, and Tolga Yuret.
\newblock An economic analysis of color-blind affirmative action.
\newblock \emph{The Journal of Law, Economics, and Organization.}, 2008.

\bibitem[Ghalme et~al.(2021)Ghalme, Nair, Eilat, Talgam-Cohen, and Rosenfeld]{pmlr-v139-ghalme21a}
Ganesh Ghalme, Vineet Nair, Itay Eilat, Inbal Talgam-Cohen, and Nir Rosenfeld.
\newblock Strategic classification in the dark.
\newblock In \emph{International Conference on Machine Learning}, 2021.

\bibitem[Haghtalab et~al.(2023)Haghtalab, Jordan, and Zhao]{haghtalab2023ondemand}
Nika Haghtalab, Michael~I. Jordan, and Eric Zhao.
\newblock On-demand sampling: Learning optimally from multiple distributions.
\newblock \emph{https://arxiv.org/abs/2210.12529}, 2023.

\bibitem[Hardt et~al.(2015)Hardt, Megiddo, Papadimitriou, and Wootters]{hardt2015Strategic}
Moritz Hardt, Nimrod Megiddo, Christos Papadimitriou, and Mary Wootters.
\newblock Strategic classification, 2015.

\bibitem[Harris et~al.(2022)Harris, Ngo, Stapleton, Heidari, and Wu]{pmlr-v162-harris22a}
Keegan Harris, Dung Daniel~T Ngo, Logan Stapleton, Hoda Heidari, and Steven Wu.
\newblock Strategic instrumental variable regression: Recovering causal relationships from strategic responses.
\newblock In \emph{International Conference on Machine Learning}. PMLR, 2022.

\bibitem[Horowitz and Rosenfeld(2023)]{a-tale-of-two-shifts}
Guy Horowitz and Nir Rosenfeld.
\newblock {A Tale of Two Shifts: Causal Strategic Classification}.
\newblock \emph{https://arxiv.org/pdf/2302.06280.pdf}, 2023.

\bibitem[Hron et~al.(2023)Hron, Krauth, Jordan, Kilbertus, and Dean]{hron2023modeling}
Jiri Hron, Karl Krauth, Michael Jordan, Niki Kilbertus, and Sarah Dean.
\newblock Modeling content creator incentives on algorithm-curated platforms.
\newblock In \emph{The {{Eleventh International Conference}} on {{Learning Representations}}}, February 2023.

\bibitem[Izzo et~al.(2021)Izzo, Ying, and Zou]{izzo2021How}
Zachary Izzo, Lexing Ying, and James Zou.
\newblock How to {{Learn}} when {{Data Reacts}} to {{Your Model}}: {{Performative Gradient Descent}}.
\newblock In \emph{Proceedings of the 38th {{International Conference}} on {{Machine Learning}}}, pages 4641--4650. {PMLR}, July 2021.

\bibitem[Izzo et~al.(2022)Izzo, Zou, and Ying]{izzo2022How}
Zachary Izzo, James Zou, and Lexing Ying.
\newblock How to {{Learn}} when {{Data Gradually Reacts}} to {{Your Model}}.
\newblock In \emph{Proceedings of {{The}} 25th {{International Conference}} on {{Artificial Intelligence}} and {{Statistics}}}, pages 3998--4035. {PMLR}, May 2022.

\bibitem[Jagadeesan et~al.(2021)Jagadeesan, Mendler-D{\"u}nner, and Hardt]{alt-micro-found-icml-2021}
Meena Jagadeesan, Celestine Mendler-D{\"u}nner, and Moritz Hardt.
\newblock {Alternative Microfoundations for Strategic Classification}.
\newblock In \emph{International Conference on Machine Learning}, 2021.

\bibitem[Jagadeesan et~al.(2022)Jagadeesan, Zrnic, and Mendler-D{\"u}nner]{jagadeesanregret2022}
Meena Jagadeesan, Tjana Zrnic, and Celestine Mendler-D{\"u}nner.
\newblock Regret minimization with performative feedback.
\newblock In \emph{International Conference on Machine Learning}, 2022.

\bibitem[Jagadeesan et~al.(2023)Jagadeesan, Garg, and Steinhardt]{Jagadeesan2023modeling}
Meena Jagadeesan, Nikhil Garg, and Jacob Steinhardt.
\newblock Supply-side equilibria in reccommender systems.
\newblock \emph{https://arxiv.org/pdf/2206.13489.pdf}, 2023.

\bibitem[Kannan et~al.(2019)Kannan, Roth, and Ziani]{KannanDownStream2019}
Sampath Kannan, Aaron Roth, and Juba Ziani.
\newblock Downstream effects of affirmative action.
\newblock In \emph{Proceedings of the Conference on Fairness, Accountability, and Transparency}, FAT* '19, page 240–248, New York, NY, USA, 2019. Association for Computing Machinery.
\newblock ISBN 9781450361255.
\newblock \doi{10.1145/3287560.3287578}.
\newblock URL \url{https://doi.org/10.1145/3287560.3287578}.

\bibitem[Karimi et~al.(2021)Karimi, Scholkopf, and Valera]{algorithmic-recourse-2021}
Amir-Hossein Karimi, Bernhard Scholkopf, and Isabel Valera.
\newblock Algorithmic recourse: from counterfactual explanations to interventions.
\newblock In \emph{ACM conference on fairness, accountability, and transparency}, 2021.

\bibitem[Kleinberg and Raghavan(2020)]{KleinbergAcm2020}
Jon Kleinberg and Manish Raghavan.
\newblock How do classifiers induce agents to invest strategically?
\newblock \emph{ACM Transactions on Economics and Computation}, 2020.

\bibitem[K{\"o}nig et~al.(2023)K{\"o}nig, Freiesleben, and Grosse-Wentrup]{improvement-recourse-2023}
Gunnar K{\"o}nig, Timo Freiesleben, and Mortiz Grosse-Wentrup.
\newblock Improvement-focused causal recourse (icr).
\newblock In \emph{AAAI Conference On Artificial Intelligence}, 2023.

\bibitem[Levanon and Rosenfeld(2021)]{strat-class-mp-2021}
Sagi Levanon and Nir Rosenfeld.
\newblock {Strategic Classification Made Practical}.
\newblock In \emph{International Conference on Machine Learning}, 2021.

\bibitem[Levanon and Rosenfeld(2022)]{Levanon-Generalized-2022}
Sagi Levanon and Nir Rosenfeld.
\newblock Generalized strategic classification and the case of aligned incentives.
\newblock In \emph{International Conference on Machine Learning}, 2022.

\bibitem[Liu et~al.(2020)Liu, Wilson, Haghtalab, Kalai, Borgs, and Chayes]{Liu2023DisparateEquilibria}
Lydia~T. Liu, Ashia Wilson, Nika Haghtalab, Adam~Tauman Kalai, Christian Borgs, and Jennifer Chayes.
\newblock The disparate equilibria of algorithmic decision making when individuals invest rationally.
\newblock In \emph{ACM Conference on Fairness, Accountability, and Transparency in Machine Learning}, 2020.

\bibitem[Liu et~al.(2022)Liu, Garg, and Borgs]{liu22strategic}
Lydia~T. Liu, Nikhil Garg, and Christian Borgs.
\newblock Strategic ranking.
\newblock In \emph{Proceedings of The 25th International Conference on Artificial Intelligence and Statistics}, 2022.

\bibitem[Maity et~al.(2021)Maity, Dutta, Terhorst, Sun, and Banerjee]{maity2021linear}
Subha Maity, Diptavo Dutta, Jonathan Terhorst, Yuekai Sun, and Moulinath Banerjee.
\newblock A linear adjustment based approach to posterior drift in transfer learning.
\newblock \emph{arXiv:2111.10841 [stat]}, December 2021.

\bibitem[Maity et~al.(2022{\natexlab{a}})Maity, Mukherjee, Banerjee, and Sun]{maity2022Predictorcorrector}
Subha Maity, Debarghya Mukherjee, Moulinath Banerjee, and Yuekai Sun.
\newblock Predictor-corrector algorithms for stochastic optimization under gradual distribution shift.
\newblock In \emph{The {{Eleventh International Conference}} on {{Learning Representations}}}, September 2022{\natexlab{a}}.

\bibitem[Maity et~al.(2022{\natexlab{b}})Maity, Yurochkin, Banerjee, and Sun]{maity2022Understanding}
Subha Maity, Mikhail Yurochkin, Moulinath Banerjee, and Yuekai Sun.
\newblock Understanding new tasks through the lens of training data via exponential tilting.
\newblock In \emph{The {{Eleventh International Conference}} on {{Learning Representations}}}, September 2022{\natexlab{b}}.

\bibitem[{Mendler-D{\"u}nner} et~al.(2020){Mendler-D{\"u}nner}, Perdomo, Zrnic, and Hardt]{mendler-dunner2020Stochastic}
Celestine {Mendler-D{\"u}nner}, Juan~C. Perdomo, Tijana Zrnic, and Moritz Hardt.
\newblock Stochastic optimization for performative prediction.
\newblock In \emph{Proceedings of the 34th {{International Conference}} on {{Neural Information Processing Systems}}}, {{NIPS}}'20, pages 4929--4939, {Red Hook, NY, USA}, December 2020. {Curran Associates Inc.}
\newblock ISBN 978-1-71382-954-6.

\bibitem[Mendler-D{\"u}nner et~al.(2022)Mendler-D{\"u}nner, Ding, and Wang]{Dunner-Ding-Wang2022}
Celestine Mendler-D{\"u}nner, Frances Ding, and Yixin Wang.
\newblock Anticipating performativity by predicting from predictions.
\newblock In \emph{Conference on Neural Information Processing Systems}, 2022.

\bibitem[Miller et~al.(2020)Miller, Milli, and Hardt]{miller2020Strategic}
John Miller, Smitha Milli, and Moritz Hardt.
\newblock Strategic {{Classification}} is {{Causal Modeling}} in {{Disguise}}.
\newblock \emph{arXiv:1910.10362 [cs, stat]}, February 2020.

\bibitem[Miller et~al.(2021)Miller, Perdomo, and Zrnic]{outside-the-echo-chamber}
John Miller, Juan~C. Perdomo, and Tijana Zrnic.
\newblock {Outside the Echo Chamber: Optimizing the Performative Risk}.
\newblock In \emph{International Conference on Machine Learning}, 2021.

\bibitem[Milli et~al.(2018)Milli, Miller, Dragan, and Hardt]{milli2018Social}
Smitha Milli, John Miller, Anca~D. Dragan, and Moritz Hardt.
\newblock The social cost of strategic classification, 2018.

\bibitem[Moro and Norman(2003)]{moro2003Affirmative}
Andrea Moro and Peter Norman.
\newblock Affirmative action in a competitive economy.
\newblock \emph{Journal of Public Economics}, 87\penalty0 (3-4):\penalty0 567--594, March 2003.
\newblock ISSN 00472727.
\newblock \doi{10.1016/S0047-2727(01)00121-9}.

\bibitem[Moro and Norman(2004)]{moro2004general}
Andrea Moro and Peter Norman.
\newblock A general equilibrium model of statistical discrimination.
\newblock \emph{Journal of Economic Theory}, 114\penalty0 (1):\penalty0 1--30, January 2004.
\newblock ISSN 00220531.
\newblock \doi{10.1016/S0022-0531(03)00165-0}.

\bibitem[Perdomo et~al.(2020)Perdomo, Zrnic, {Mendler-D{\"u}nner}, and Hardt]{perdomo2020Performative}
Juan Perdomo, Tijana Zrnic, Celestine {Mendler-D{\"u}nner}, and Moritz Hardt.
\newblock Performative {{Prediction}}.
\newblock In \emph{Proceedings of the 37th {{International Conference}} on {{Machine Learning}}}, pages 7599--7609. {PMLR}, November 2020.

\bibitem[Phelps(1972)]{phelps1972}
Edmund Phelps.
\newblock The statistical theory of racism and sexism.
\newblock \emph{The American Economic Review}, 1972.

\bibitem[Shavit et~al.(2020)Shavit, Edelman, and Axelrod]{shravit-caus-strat-LR-2020}
Yonadav Shavit, Benjamin Edelman, and Brian Axelrod.
\newblock Causal strategic linear regression.
\newblock In \emph{International Conference on Machine Learning}, 2020.

\bibitem[Virtanen et~al.(2020)Virtanen, Gommers, Oliphant, Haberland, Reddy, Cournapeau, Burovski, Peterson, Weckesser, Bright, {van der Walt}, Brett, Wilson, Millman, Mayorov, Nelson, Jones, Kern, Larson, Carey, Polat, Feng, Moore, {VanderPlas}, Laxalde, Perktold, Cimrman, Henriksen, Quintero, Harris, Archibald, Ribeiro, Pedregosa, {van Mulbregt}, and {SciPy 1.0 Contributors}]{2020SciPy-NMeth}
Pauli Virtanen, Ralf Gommers, Travis~E. Oliphant, Matt Haberland, Tyler Reddy, David Cournapeau, Evgeni Burovski, Pearu Peterson, Warren Weckesser, Jonathan Bright, St{\'e}fan~J. {van der Walt}, Matthew Brett, Joshua Wilson, K.~Jarrod Millman, Nikolay Mayorov, Andrew R.~J. Nelson, Eric Jones, Robert Kern, Eric Larson, C~J Carey, {\.I}lhan Polat, Yu~Feng, Eric~W. Moore, Jake {VanderPlas}, Denis Laxalde, Josef Perktold, Robert Cimrman, Ian Henriksen, E.~A. Quintero, Charles~R. Harris, Anne~M. Archibald, Ant{\^o}nio~H. Ribeiro, Fabian Pedregosa, Paul {van Mulbregt}, and {SciPy 1.0 Contributors}.
\newblock {{SciPy} 1.0: Fundamental Algorithms for Scientific Computing in Python}.
\newblock \emph{Nature Methods}, 17:\penalty0 261--272, 2020.
\newblock \doi{10.1038/s41592-019-0686-2}.

\bibitem[Williams(1992)]{Williams-Reinforce-1992}
Ronald~J. Williams.
\newblock Simple statistical following algorithms for connectionist reinforcement learning.
\newblock \emph{Machine Learning}, 1992.

\bibitem[Wood et~al.(2021)Wood, Bianchin, and Dall'anese]{Wood2021}
Killian Wood, Gianluca Bianchin, and Emiliano Dall'anese.
\newblock Online projected gradient descent for stochastic optimization with descision dependent distributions.
\newblock \emph{IEEE Control Systems Letters}, 2021.

\bibitem[Yu et~al.(2022)Yu, Yang, and Fan]{Yu-Yang-Fan-2022}
Mengxin Yu, Zhuoran Yang, and Jianqing Fan.
\newblock Strategic decision-making in the presence of information asymmetry: Provably efficient reinforcement learning with algorithmic instruments.
\newblock \emph{https://arxiv.org/pdf/2208.11040.pdf}, 2022.

\bibitem[Zrnic et~al.(2021)Zrnic, Mazumdar, Sastry, and Jordan]{zrnic2021who}
Tijana Zrnic, Eric Mazumdar, Shankar Sastry, and Michael Jordan.
\newblock Who leads and who follows in strategic classification?
\newblock In \emph{Advances in Neural Information Processing Systems}, 2021.

\end{thebibliography}
\newpage
\appendix
\section{Stochastic optimization for reverse-causal strategic learning}
In this section, we develop a stochastic optimization algorithm to minimize the performative risk \eqref{eq: anti-causal-objective}. Previously, the work \citet{izzo2021How} gave an algorithm for performative optimization under the assumption that the performative map is of the form $\cD(\theta)=p(z; f(\theta))$, with only $f$ unknown. Our work is in a distinct environment, where the learner has an apriori model for the map $\cD$ but the base distribution is unknown. This is similar to the assumptions in \citet{strat-class-mp-2021}, which tackles the optimization for non-causal strategic classification. 
\subsection{Learning Set Up}
In practice, learning in reverse causal strategic environments is done sequentially (this is similar to the assumed in other performative prediction works). At each round of learning the learner publishes a decision $\theta$ and all data received in that round are drawn via $\cD(\theta)$. Succinctly, in a stochastic (or finite data) setting the order of learning at time is as follows:
\begin{enumerate}
\item The firm deploys decision $\theta_t$ 

\item Firm receives data (and corresponding loss/reward) drawn from 
$\cD(\theta_t)$. The firm may use this to either deploy a reactive decision or in some algorithm that converges towards an optimal policy (for example in the methodology illustrated below).
\end{enumerate}
Our methodology will be applicable in this stochastic setting, the two strongest assumptions we need are as follows:
\begin{enumerate}
    \item As in previous works on optimization for strategic classification (\cite{strat-class-mp-2021}), the learner is fully cognisant of the reverse causal strategic map $Y_+(y, \theta)$.
    \item The learner has knowledge of the feature generating mechanism $\phi(x|y)$. The justification here is that in many reverse causal strategic settings $\phi(x|y)$ is based on some standards that the learner sets (eg a hiring firm giving an interview) and thus the learner should have knowledge of this mechanism 
\end{enumerate}
Performative learning in a scenario where $Y_+(y, \theta)$ is unknown is an open and interesting problem. If $\phi(x \mid y)$ is unknown, it can be estimated from data pairs $(x,y)$.
\subsection{Methodology}
We let $\phi(x|y)$ denote the conditional density of $x$ given $y$. We also assume that we deploy some parametric model $f_{\theta}(x)$, and write $Y_+(f_{\theta}, y) \triangleq G_{\theta}(y)$.
In this notation, the cost function in the anti-causal objective can be rewritten as (recall $\cZ = \cX \times \cY$)
\[\textstyle
L(\theta) = \int_{\cZ} \ell(f_\theta(x), G_{\theta}(y)) \phi(x|G_{\theta}(y))p(y)dz.
\]
We derive the gradient in the following lemma.
\begin{lemma}
    The gradient of $L(\theta)$ is $\nabla L(\theta) = \nabla_1 L(\theta) + \nabla_2 L(\theta) + \nabla_3 L(\theta)$, where
    \begin{equation}
\label{eq: Total Grad}
\begin{aligned}\textstyle
\nabla_1L(\theta) =   \int_{\cZ}\partial_\theta[f_\theta(x)] \partial_1[\ell(f_\theta(x), G_{\theta}(y_i))] \phi(x|G_{\theta}(y_i))dz \\\textstyle
\nabla_2L(\theta) =  \int_{\cZ} \partial_\theta G_{\theta}(y_i)\partial_2[\ell(f_\theta(x), G_{\theta}(y_i))] \phi(x|G_{\theta}(y_i))dz\\\textstyle
\nabla_3L(\theta) = \int_{\cZ}\partial_{\theta}G_{\theta}(y_i)\ell(f_\theta(x), G_{\theta}(y_i)) \partial_2[\phi(x|G_{\theta}(y_i))]dz,
\end{aligned}
\end{equation}
and $\partial_{\theta} G_{\theta}(y)$ is the solution to the following:
\begin{equation}\textstyle
\label{eq: Argmax-Grad}
\int_{\cX}\partial_{\theta}[f_{\theta}(x)]\phi(x|G_{\theta}(y))dx +(\int_{\cX}f_{\theta}(x)\partial_{2}^2[\phi(x|G_{\theta}(y))]dx - \partial_{1}^2 c( G_{\theta}(y),y)) \partial_{\theta}G_{\theta}(y) = 0.
\end{equation}
\end{lemma}

We propose that the learner utilize Monte-Carlo techniques to numerically approximate the integrals of \ref{eq: Argmax-Grad} and \ref{eq: Total Grad}. One applicable option is to use a REINFORCE approximation (see \cite{izzo2021How} or \cite{Williams-Reinforce-1992}). This technique allows us to re-write the integrals of \ref{eq: Total Grad}, \ref{eq: Argmax-Grad},
 \[\textstyle\int_{\cX}f_{\theta}(x)\partial_{2}^2[\phi(x|G_{\theta}(y))]dx = \int_{\cX} f_{\theta}(x) \partial_2^2 [\text{log}(\phi(x)|G_{\theta}(y))]\phi(x|G_{\theta}(y))dx \]
 \[= \mathbb{E}_{x\sim \phi(x|G_{\theta}(y))} [f_{\theta}(x) \partial_2^2 [\text{log}(\phi(x)|G_{\theta}(y))]],\]
 \[\textstyle\int_{\cX}\ell(f_\theta(x), G_{\theta}(y_i)) \partial_2[\phi(x|G_{\theta}(y_i))]dx = \int_{\cX}\ell(f_\theta(x), G_{\theta}(y_i)) \partial_2[\text{log}(\phi(x|G_{\theta}(y_i)))]\phi(x|G_{\theta}(y_i))dx\]
 \[= \mathbb{E}_{x\sim \phi(x|G_{\theta}(y_i))}\ell(f_\theta(x), G_{\theta}(y_i)) \partial_2[\text{log}(\phi(x|G_{\theta}(y_i)))].\]
 Since the learner knows $\phi$, these expressions can be computed via drawing samples from $\phi(x|G_{\theta}(y))$.
 The other integrals in \ref{eq: Argmax-Grad} and \ref{eq: Total Grad} are readily computable via this method in their vanilla forms. We summarize the proposed algorithm for the performative prediction problem in Algorithm \ref{alg:cap}.
% \subsection{Theoretical Guarantees}

% \YK{state two results for SGD convergence; one where you assume convexity and one where you don't. State you'll check assumptions for the labor market ex later in S4}

\begin{algorithm}[H]
%\setstretch{1.25}
\caption{Reverse Causal Strategic SGD}\label{alg:cap}
\begin{algorithmic}
\State $\textit{Input: } u, \eta, n, \theta_0$
% \State $\textit{Initialize: } \theta = \theta_0$
\While{not converged}
\State $\textit{Draw } \{y_i\}_{i=1}^n \sim p$ and observe $G_{\theta}(y_i)$
\State Draw $\{x_j^i\}_{j=1}^{n'}$ from $\phi(x|G_{\theta}(y_i))$
\State Compute $\hat{\nabla}L(\theta)(y_i, \{x_j^i\}_{j=1}^{n'})$ with REINFORCE
\State $\theta_t \gets \theta_{t-1} - \eta \sum_i \hat{\nabla}L(\theta)(y_i, \{x_j^i\}_{j=1}^{n'})$
\EndWhile
% \State \Return $\theta_T$
\end{algorithmic}
\end{algorithm}
Figure \ref{Fig: alg1-vs-rsgd} gives an example of this algorithm utilized in the hiring of strategic workers.

\subsection{Proof of Lemma A.1}

The only tricky term to evaluate is $\partial_{\theta} G_{\theta}(y)$. To deal with problem we utilize the implicit function theorem. We have
\[
\partial_{\theta} G_{\theta}(y) = \partial_{\theta} \textit{ argmax}_{y'} \int_{\cX} f_{\theta}(x)\phi(x|y')dx - c(y', y).
\]
Via the implicit function theorem $\partial_{\theta} G_{\theta}(y)$ must solve

\[
\partial_{\theta}[\partial_y (\int_{\cX} f_{\theta}(x)\phi(x|y')dx - c(y', y))=0].
\]
Which implies that $\partial_\theta G_{\theta} (y)$ is given by:
\begin{equation}
\int_{\cX}\partial_{\theta}[f_{\theta}(x)]\phi(x|G_{\theta}(y))dx +(\int_{\cX}f_{\theta}(x)\partial_{2}^2[\phi(x|G_{\theta}(y))]dx - \partial_{1}^2 c( G_{\theta}(y),y)) \partial_{\theta}G_{\theta}(y) = 0.
\end{equation}

\section{Market Formulation and Experiments}
\subsection{Market Details}
This market is conceptually similar to that of \cite{coate1993Will}, with the twist being that worker skill $y$ is now continuous, with $y \in \cY \subset \mathbb{R}$ generated from $p(y)$. The employer still makes hires based on noisy skill assessment $X$ generated from $\Phi(x|y)$ and group membership $g \sim \text{Ber}[\lambda]$. Employer production is specified by a utility function $u$ which depends on the skill level of a worker and the probability that worker is hired; $u: \mathbb{R} \times [0,1] \rightarrow \mathbb{R}$. Letting $H_{f}(y) = \lambda H_{f}^{\text{Maj}}(y) + (1-\lambda) H_{f}^{\text{Min}}(y)$ with $H_{f}^i(y) \triangleq \int_{\cX} f_i(x)d\Phi(x|y)$, the employer's (non performative) policy maximization objective is given by  
\[\text{max}_{f \in \cF} \: \mathbb{E}[u(y, H_{f}(y))].\]

A hired worker is given wage $w(x ,f)$, dependent on the observed skill level and policy choice and they can improve skills at a cost $c(\cdot, \cdot) \rightarrow \mathbb{R}^{+}$; thus workers strategically choose outcomes by selecting 
\[Y_+(y, f) = \argmax_{y'} \: \int_{\cX} w(x, f_i) f_i(x) d\Phi(x|y') - c(y', y).\]
Using the same notation as in section 2, the employer's post strategic response problem is given by 
\[\text{max}_{f \in \cF} \: \mathbb{E}[u(Y_+(y, f), H_{f}(Y_+(y, f)))].\]
In comparing worker well-being under stable and optimal policies, we first must quantify the well-being of the workers. As before, one metric we use is the proportion of qualified workers, which in this case is the probability a worker drawn from a distribution induced by a policy is skilled enough to provide positive employer utility.

We also introduce a second metric, the cost adjusted worker welfare. Towards this, we define the welfare of a worker with a given natural skill $y$ and group membership $i$ as their net benefit in the market post policy deployment and strategic update. Specifically, this is given by
\[\textstyle W_i(f, y) = \int_{\cX} w(x, f_i) f_i(x) d\Phi(x|Y_+(y, f_i)) - c(Y_+(y, f_i), y).\]
We will define the aggregate worker welfare as the expectation of $W_i (f, y)$ with respect to $p(y)$.

To study the techniques discussed in appendix A, we must focus on the case when the optimal policy $f$ is in a parametric policy class, \ie\ $\cF = f(x,\theta); \: \theta \in \mathbb{R}$. In the appendix, we provide sufficient conditions for this to hold. 
\begin{proposition}
Assume the ratio $\phi(x|y_1)/\phi(x|y_0)$ is increasing in $x$ for all $y_1 > y_ 0$, and the utility $u(y)$ is increasing in $y$.  Then the optimal policy choice for the employer is some threshold policy $f(x) = 1_{x>\theta}$.
\end{proposition}
\begin{proof}
 We denote the marginal density of the signal $x$ as $\nu(x)$ and subscript any distribution that depends on the policy with $f$; also WLOG assume $\cY = [0,1]$.  We can write the employers (non-performative) utility as the following:

$$U(f)= \int f (x) \nu(x) \int u(y) p(y|x)dydx$$

Thus the optimal policy is of the form:
$$f(x) = 1\{\int u(y) p(y|x)dy > 0\}$$
Thus the lemma statement is equivalent to $\int u(y) py|x)dy$ being monotone in $x$.
Using integration by parts we re-write this as the following:
$$\int u(y) p(y|x)dy = u(y) P(Y \leq y| X = x)|_{0}^{1} - \int u'(y) P(Y \leq y |X =x) dy$$
$$= u(1)*1 - u(0)*0 - \int u'(y) P_{f}(Y \leq y |X =x) dy$$
Since $u'(y) > 0$ by assumption; we need to show that $P(Y \leq y |X =x) $ is decreasing in $x$ for all $y$. We have the following:
$$P(Y \leq y |X =x) = \int_{0}^{y} P_(Y = y| X = x)dy = \frac{\int_{0}^{y}\phi(x|y)p(y)dy}{\int_{0}^{1}\phi(x|y)p(y)dy}$$

$$=  1 / (1 + \frac{\int_{0}^{y}\phi(x|y)p(y)dy}{\int_{y}^{1}\phi(x|y)p(y)dy})$$
By assumption $y_0<y_1$ the ratio $\phi(x|y_0)/\phi(x|y_1)$ is decreasing in $x$.    Thus for all $y$ and any distribution $p$ the ratio $\int_{0}^{y}\phi(x|y)p(y)dy/{\int_{y}^{1}\phi(x|y)p(y)dy}$ is monotonically decreasing in $x$.
\end{proof}

\subsection{Non-Linear Utility and Wages}
In several alternative labor market models \citep{arrow1971, moro2004general}, wage structures are not flat; instead they are determined by competition between competing employers. Additionally, the authors of \citet{moro2004general} find that moving beyond linear utility functions introduces inter-group interaction to the labor market, which is a key to revealing sources of discrimination.

Throughout this section we will assume that $u(y, H_{\pi}(y)) = \gamma(H_{\pi}(y))u(y)$ for strictly increasing functions $\gamma: [0,1] \rightarrow [0,1]$; $u: \mathbb{R} \rightarrow \mathbb{R}$; and additionally that $\Pi(x; \theta) = 1_{x>\theta}.$

\textbf{Limitations of Linear Utility:}

We briefly discuss the limitations of the model if the production function is linear, \ie\ $\gamma(\cdot)$ is the identity function. In such a case, via the linearity of expectation, one can write the vanilla employer utility objective as
\[\text{max}_{\pi} \lambda \mathbb{E}H_{\pi}^0(y)u(y) + (1-\lambda)
\mathbb{E}H_{\pi}^1(y)u(y). \]
Thus, in this case, the employer can simply pick an optimal policy $\pi_i$ for each group separately. Performative stable policies with which $\pi_1(x) \neq \pi_0 (x)$ may exist; but one group does not benefit from discrimination of another group. This lack of interaction is why we go beyond the linear utility case.

\textbf{Determination of Wage Structure:}

We adhere to the principle (which is shown rigorously in \citet{moro2004general}) that at a Nash equilibrium among two or more competing employers, the net profits of each employer should be zero. Equivalently: $\mathbb{E}[\gamma(H_{f}(y))u(y)] - \mathbb{E}[w(x;f)] = 0$. Thus, the wage offered to a worker with signal $x$ and group membership $i$ should be the expected utility of such a worker conditioned on these traits. We can write this for an individual in group $i$ as
\[\textstyle\mathbb{E}[\gamma(h_{f_i}(y))u(y)|X=x] = \mathbb{E}[\gamma(f_i(x))u(y)|X = x] = \gamma(f_i(x))\frac{\int_{\mathbb{R}}u(y)\phi(x|y)dp(y)}{\int_{\mathbb{R}}\phi(x|y)dp(y)}.\]
As such, for this section we will assume the wages proffered by the employer are
\begin{equation}\textstyle
    \label{eq: Nash-wage-structure}
    w_i(f, x) = \gamma(f_i(x))\frac{\int_{\mathbb{R}}u(y)\phi(x|y)dp(y)}{\int_{\mathbb{R}}\phi(x|y)dp(y)}.
\end{equation}
\begin{example}
    As an example of this assumed wage structure, consider a market specified by $u(y) = y$, $p(y) = \cN(0, \sigma_{y}^2), f(x|y) = \cN(y, \sigma_x^2)$. In this market, a employer (with hiring threshold $\theta$) offers wage
    \[\textstyle w(x, \theta) = \frac{x}{(1+\sigma_x^2/\sigma_y^2)} \mathbf{1}[x\geq \theta].\]
    Wages scale linearly with perceived skill of a worker; additionally, they increase as the profitability-noise ratio increases (a decrease in $\sigma_x^2/\sigma_y^2$).
\end{example}
\begin{proof}
The wage structure is given by: 
\[
1_{x\geq \theta}\frac{\int_\reals y e^{-(x-y)^2/2\sigma_x^2} e^{-y^2/2\sigma_y^2}dy}{\int_\reals e^{-(x-y)^2/2\sigma_x^2} e^{-y^2/2\sigma_y^2}dy}
 = 1_{x \geq \theta}\frac{\int_\reals y e^{(-1/2\sigma_x^2-1/2\sigma_y^2)(y-x/\sigma_x^2(1/\sigma_x^2+1/\sigma_y^2))^2}dy}{\int_\reals  e^{(-1/2\sigma_x^2-1/2\sigma_y^2)(y-x/\sigma_x^2(1/\sigma_x^2+1/\sigma_y^2))^2}dy}\]
 We can multiply the top and bottom row by the needed normalization constants, then the top row is the expectation of a normal random variable with mean $x/(1+\sigma_x^2/\sigma_y^2))$ and the bottom row integrates to one.
 \end{proof}
For Figure \ref{Fig: Nonlin-Worker-Welfare} we stick with this example of a market; additionally assuming that the cost is again quadratic $c(y', y) = \frac{c}{2}(y'-y)^2$, and that $\gamma(H_{f}(y)) = (H_{f}(y))^\alpha$. 

\subsection{Additional Experiments}
In figure \ref{Fig: LinearFairness} we present an example of a market that demonstrates the phenomena discussed in Theorems \ref{thm:equity-effect} and \ref{thm:equity-effect-2}. The left hand side plots two curves: the dotted one represents proportion of qualified workers resulting from a hiring policy, while the solid one is employers best response as a function of the proportion of qualified workers; intersections of these correspond to stable policies. There are two stable policies with an appreciable gap in worker qualification; simultaneously the performative utility has a unique maximum and thus all pairs of optimal policies are fair.
\begin{figure}[H]
\begin{centering}
  \begin{subfigure}[b]{0.45\textwidth}
    \includegraphics[width=\textwidth]{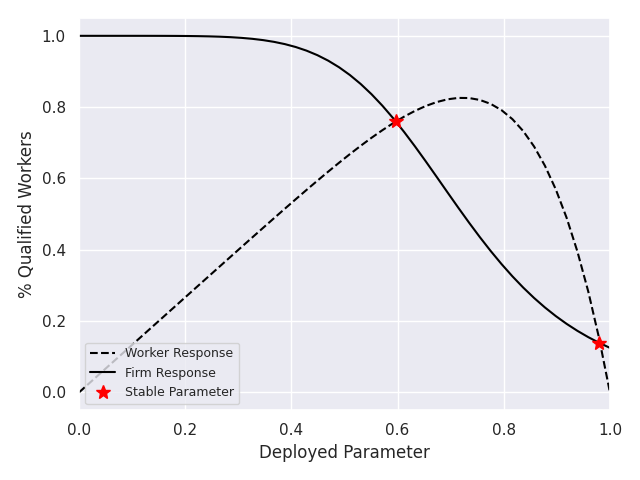}
    %\caption{employer Utility vs Round }
  \end{subfigure}
  \begin{subfigure}[b]{0.45\textwidth}
    \includegraphics[width=\textwidth]{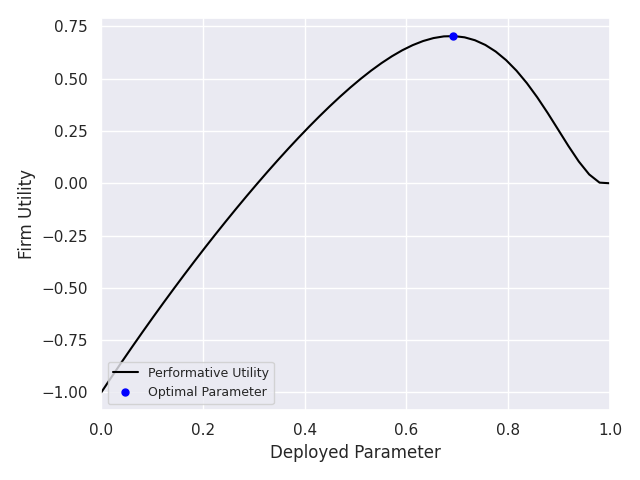}
    %\caption{Policy Deployments vs Round }
  \end{subfigure}
    % \vspace{-1.5pc}
  \caption{A market with fair optimal pairs and discriminatory stable pairs}
  \label{Fig: LinearFairness}
  \end{centering}
\end{figure}
Here is another demonstration of Algorithm 1 in the context of labor markets.
\begin{figure}[H]
\begin{center}
  \begin{subfigure}[b]{0.43\textwidth}
    \includegraphics[width=\textwidth]{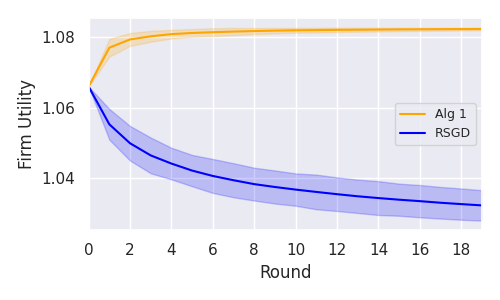}
    %\caption{employer Utility vs Round }
  \end{subfigure}
  \begin{subfigure}[b]{0.43\textwidth}
    \includegraphics[width=\textwidth]{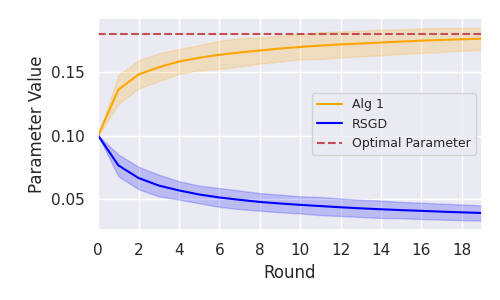}
    %\caption{Policy Deployments vs Round }
  \end{subfigure}
    % \vspace{-1.5pc}
  \caption{Algorithm 1 and RSGD with $a = 4, \: w = 1, \: c = 4$}
  \label{Fig: alg1-vs-rsgd}
\end{center}   
\end{figure}
\subsection{Experimental Details}
\textbf{Section 4.2}
\begin{enumerate}
    \item Figure \ref{Fig: alg1-vs-RSGD-2} (left):  Experiment is run on a market parameterized by $y_i \sim \textit{Unif}[0,1]$, $\phi(x|y) = (y+1)x^y 1_{0<x<1}$, $u(y) = ay -1$, $c(y, y') = \frac{c}{2}(y'-y)^2$ ($a = 2, \: w = 1, \: c = 5$). A total population of 1000 $y$ were drawn and at each round $100$ are sampled, learning was run for $20$ rounds. 10 separate trials were run, with darker lines showing the mean and the shade indicating a 90 percent confidence interval. Step size $(0.1/(t+1))$ is used for both RSGD and Algorithm 1 iterations, with an initial seed of $0.1$ for each. 
    \item Figure \ref{Fig:Lin-Worker-Welfare} and figure \ref{Fig: alg1-vs-RSGD-2} (right): The same market configuration is used (specific parameters are indicated in captions). Optimal policies were located using sci-py (\cite{2020SciPy-NMeth}) minimization packages. To locate stable points RRM was run until iterations were within a distance of $0.001$ of one another. 
    \item Figure \ref{Fig: LinearFairness}: For simplicity, a simpler market parametrization is used with $u(y) = 1_{\{y \geq 0.5\}} - 1_{\{y \leq 0.5\}}$, $\phi(x|y)= 2x 1_{\{y > 0.5\}} + 1* 1_{\{y<0.5\}}$, $p(y) \sim U[0, 0.5]$, $c(y, y') = c(y-y')_{+}$
\end{enumerate}
\textbf{Appendix B} 

\begin{enumerate}
    \item Figure \ref{Fig: LinearFairness}: $u(y) = 15y, p(y) = \cN(0,1), \phi(x|y) = \cN(y, 1), c(y', y) = \frac{c}{2}(y'-y)^2$. For the top row group costs are set as $c_{\text{Min}} = 25/(1-\lambda); c_{\text{Maj}}=20/(\lambda)$. In the bottom row group costs are set at $c_\text{Min} = 25, c_\text{Maj} = 20$ or $c_\text{Min} = 40, c_\text{Maj} = 20$, while $\lambda = 0.8$. Optimal and stable policies are located in a similar manner, with RRM being terminated after iterations are within $0.00001$ of one another for stable polices.
    \item  Figure \ref{Fig: alg1-vs-rsgd}: Experiment is run on a market parameterized by $y_i \sim \textit{Unif}[0,1]$, $\phi(x|y) = (y+1)x^y 1_{0<x<1}$, $u(y) = ay -1$, $c(y, y') = \frac{c}{2}(y'-y)^2$ ($a = 4, \: w = 1, \: c = 4$). A total population of 1000 $y$ were drawn and at each round $100$ are sampled, learning was run for $20$ rounds. 10 separate trials were run, with darker lines showing the mean and the shade indicating a 90 percent confidence interval. Step size $(0.1/(t+1))$ is used for both RSGD and Algorithm 1 iterations, with an initial seed of $0.1$ for each. 
\end{enumerate}
\section{Section 3 Proofs}
\subsection{Proof of Theorem 3.1}
\begin{proof}
For notational convenience we refer to the performative employer utility as $U$. We also denote $\text{TPR}(\theta) = P(X > \theta| y = 1)$ and $\text{FPR}(\theta) = P(X > \theta | y=0)$ as the true postive rate (resp. false positive rate) of the deployed classifier. The (decoupled) performative learners utility is the following:
\[ U(\theta_1, \theta_2) \triangleq \: \:  p_+ G(w[\text{TPR}(\theta_2) - \text{FPR}(\theta_2)]) \text{TPR}(\theta_1) - p_- (1-G(w[\text{TPR}(\theta_2) - \text{FPR}(\theta_2)])) \text{FPR}(\theta_1)\]
Consider differentiating this wrt the first argument:
$$\partial_1 U(\theta_1, \theta_2) = \partial_{\theta_1} [p_+ \pi(\theta_2)\int_{\theta_1}^1d\Phi(x\mid1)- p_-(1-\pi(\theta_2))\int_{\theta_1}^1 d\Phi(x|0)]$$
$$\stackrel{\textit{FTC}}{=}-p_+\pi(\theta_2)\phi(\theta_1\mid 1) + p_-(1-\pi(\theta_2))\phi(\theta_1\mid 0)$$
By the definition of stability (stable points are optimal on their induced distribution, and end points will not be optimal) any stable point $\theta_s$ must satisfy $\partial_1 U(\theta_1, \theta_2) \mid_{\theta_1=\theta_2 = \theta_s} = 0$.
Thus we must have:
\[-p_+ \pi(\theta_s)\phi(\theta_s|y=1)+p_-(1-\pi(\theta_s))\phi(\theta_s|y=0) = 0\]
Or equivalently:
\[\frac{p_-(1-\pi(\theta_s))}{p_+\pi(\theta_s)} = \frac{\phi(\theta_s|1)}{\phi(\theta_s|0)}\]
From here we apply the assumption that $\phi(x|1)/ \phi(x|0) > \delta_2$: 
$$\frac{p_-(1-\pi(\theta_s))}{p_+ \pi(\theta_s)} > \delta_2 \implies p_-(1-\pi(\theta_s)) > \delta_2 p_+ \pi(\theta_s)$$
$$\implies p_-\delta_2p_+\pi(\theta_s)+p_-\pi(\theta_s) \implies \frac{p_-}{p_+ \delta_2 + p_-} > \pi(\theta_s).$$
So any stable point $\theta_s$ satisfies $\pi(\theta_s) < \frac{p_-}{p_+ \delta_2 + p_-}$. The intuition from here is straightforward, if $p_+$ is large then $\pi(\theta_s)$ must be small but the optimal labor force skill level for a learner should not be small. To formalize this, consider plugging in the upper bound for $\pi(\theta_s)$ to the learners objective. Any $\theta$ which satisfies the bound will have:
\[U(\theta) \leq \frac{p_+p_-}{p_+\delta_2+p_-} \text{TPR}(\theta) \leq \frac{p_-}{\delta_2}\]
On the other hand consider $U(\tilde{\theta})$ (with $\tilde{\theta}$ defined as in assumption 1.)  Assume that $w$ is large enough so that $w \delta_1 > M_g$, then we have the following:
\[U(\tilde{\theta}) = p_+ * 1 * \text{TPR}(\tilde{\theta}) - p_- * 0 * \text{FPR}(\tilde{\theta}) > p_+ * (\delta_1+\text{FPR}(\tilde{\theta}))\]
    Now clearly if $p_+ > p_-/(\delta_1 \delta_2)$  it holds that $p_+ (\delta_1+\text{FPR}(\tilde{\theta})) > p_-/\delta_2$.  In this case, since $U(\tilde{\theta}) > U(\theta)$ for all $\theta$ such that $\pi(\theta) < p_-/(p_+\delta_2 +p_-)$ it can not hold that $\pi(\theta_{\text{opt}}) < p_-/(p_+\delta_2+p_-)$; this completes the proof of the first part.
To prove the second part note that $U(\theta_{\text{opt}}) \geq U(\tilde{\theta}) \geq p_+ \delta_1$. Thus:
\[U(\theta_{\text{stab}}) \leq \frac{p_- U(\theta_{\text{opt}})}{\delta_1(p_+\delta_2+p_-)}\]
Since by assumption  $p_+ \delta_1 \delta_2 > p_-$ we have $U(\theta_{\text{stab}}) < p_- U(\theta_{\text{opt}})/(p_-+p_-\delta_1)$ which completes the proof.
\end{proof}
\subsection{Proof of Theorem 3.2}
\begin{proof}
An identical argument as before will show that any stable point $\theta_s$ satisfies $\pi(\theta_s) < \frac{p_-}{p_+ \delta_2 + p_-}$ which in turn implies that for all stable points $U(\theta_s) < p_-/ \delta_2$. Now consider $U(\tilde{\theta})$. We have the following: 
\[U(\tilde{\theta}) = p_+ \pi(\tilde{\theta}) \text{TPR} (\tilde{\theta}) - p_- (1-\pi(\tilde{\theta})) \text{FPR}(\tilde{\theta})\]
By the assumptions of the theorem, we have that $ p_+ \pi(\tilde{\theta}) > p_- (1-\pi(\tilde{\theta}))$, and thus $U(\tilde{\theta}) > \delta_1$. Since $\delta_1 > p_- / \delta_2$, for all $\theta_s$, $U(\tilde{\theta}) > U(\theta_s)$, and in particular no optimal point $\theta_{\text{opt}}$ can have $\pi(\theta_{\text{opt}}) <  \frac{p_-}{p_+ \delta_2 + p_-}$.

The statement that $U_{\text{perf}}(\theta_{\text{stab}}) \leq U_{\text{perf}}(\theta_{\text{opt}})$ follows directly from the definition of optimal points.
\end{proof}
\subsection{Proof of Theorem 3.3}
\begin{proof}
We first prove the statement on stable points.
Let $\theta^*(\pi)$ denote the optimal policy for the vanilla utility if the proportion of qualified workers is $\pi$. Note that $\theta^*(\pi)$ is a monotonically decreasing function of $\pi$. The existence of the stable set $\vec{\theta}_{\text{stab}}(\theta^\text{Maj}_{\text{stab}}, \theta^\text{Min}_{\text{stab}})$ such that $\pi^{\text{Maj}}(\vec{\theta}_{\text{stab}}) \neq \pi^{\text{Min}} (\vec{\theta}_{\text{stab}})$ is equivalent to the existence of multiple intersections of the following functions from $\Theta = [0,1]$ to $ [0,1]$: $$f_1(\theta) =  \theta^{* -1}(\theta),$$
$$f_2(\theta) =  G(w(P(x>\theta|y=1))-P(x>\theta|y=0))$$

Let $Z(\theta) = f_1(\theta) - f_2(\theta)$. Note that $Z(0) = 1 - 0 > 0$ and that $Z(1) = 0 - 0 = 0$. Additionally, consider $Z(\tilde{\theta})$. By the assumptions on $w$, $f_2(\tilde{\theta}) = 1$, and since $f_1(\theta)$ is strictly decreasing in $\theta$ $f_1(\tilde{\theta}) < 1$ so $Z(\tilde{\theta}) < 0$. By the assumptions $Z$ is continuous and thus by IVT $Z$ has at least one additional zero on $[0,1]$. WLOG let $\theta_{\text{stab}}^\text{Min} = 1$ so that $\pi(\theta_{\text{stab}}^\text{Min}) = 0$, and let $\theta_{\text{stab}}^\text{Maj}$ be the other stable point whose existence we have just proved. Note that $\theta_{\text{stab}}^\text{Maj} < \tilde{\theta}$ so by the c-Lipschitz property of $f_1(\theta)$ we have the following:
\[|f_1(\theta_{\text{stab}}^\text{Maj}) - f_1(0)| = |f_1(\theta_{\text{stab}}^\text{Maj}) - 1| \leq c\theta_{\text{stab}}^{\text{Maj}} \leq c\tilde{\theta}.\]
Since $f_1(\theta_{\text{stab}}^\text{Maj}) = \pi(\theta_{\text{stab}}^\text{Maj})$ we have $\pi(\theta_{\text{stab}}^\text{Maj}) > 1-c\tilde{\theta} > 1 - c$ which completes part 1 of the proof.

We now complete the second part of the proof. Plugging in $\tilde{\theta}$ to the performative utility and using the assumptions on $G$, $\epsilon$, $w$ we have
\[U(\tilde{\theta}) > p_+ *1 *(1-\epsilon) - p_- *0 * \epsilon = p_+ (1-\epsilon) \]
Since any optimal point $\theta_{\text{opt}}$ satisfies $U(\theta_{\text{opt}}) \geq U(\tilde{\theta})$ we know that for all optimal points:
\[p_+ \pi(\theta_{\text{opt}})  P(x>\theta_{\text{opt}}| y=1)- p_- (1-\pi(\theta_{\text{opt}})) P(x>\theta_{\text{opt}}| y=0) > p_+(1-\epsilon) \]
\[\implies \pi(\theta_{\text{opt}})[P(x>\theta_{\text{opt}}\mid y=1) + P(x>\theta_{\text{opt}}\mid y=0)] - P(x>\theta_{\text{opt}}\mid y=0) > p_+(1-\epsilon) \]
\[\implies \pi(\theta_{\text{opt}}) > \frac{1-\epsilon}{P(x>\theta_{\text{opt}}\mid y=1) + P(x>\theta_{\text{opt}}\mid y=0)}+\] 

\[ \frac{ P(x>\theta_{\text{opt}}\mid y=0)}{P(x>\theta_{\text{opt}}\mid y=1) + P(x>\theta_{\text{opt}}\mid y=0)} \]
Since $0 \leq P(x>\theta_{\text{opt}}\mid y=1) \leq 1 $ and $0 \leq P(x>\theta_{\text{opt}}\mid y=0)$, a simple calculation will show
\[\frac{1-\epsilon}{P(x>\theta_{\text{opt}}\mid y=1) + P(x>\theta_{\text{opt}}\mid y=0)}+\] 
\[ \frac{ P(x>\theta_{\text{opt}}\mid y=0)}{P(x>\theta_{\text{opt}}\mid y=1) + P(x>\theta_{\text{opt}}\mid y=0)} > 1-\epsilon.\]
This in turn implies $\pi(\theta_{\text{opt}}) > 1-\epsilon$ for any optimal point.  Thus all optimal pairs $\vec{\theta}_{\text{opt}} = (\theta^\text{Maj}_{\text{opt}}, \theta^{\text{Min}}_{\text{opt}})$ must satisfy $|\pi^{\text{Maj}}(\vec{\theta}_{\text{opt}}) -\pi^{\text{Min}}(\vec{\theta}_{\text{opt}})| < \epsilon$
\end{proof}
\begin{remark}

It is generally possible to construct markets in which both $\epsilon$ and $c$ are small. In fact, the derivative of $\theta^{* -1}(\theta)$ is given by
\[[\theta^{* -1}(\theta)]' = \frac{\theta^{* -1}(\theta) \phi'(\theta|1) + (1-\theta^{* -1}(\theta))\phi'(\theta|0)}{\phi(\theta|1) + \phi(\theta|0)}.\]
Thus, for example any market for which $\phi'(x\mid 1) \approx 0$ for $x < \tilde{\theta}$ and $\phi(x|0) = 1$ can satisfy both. See Figure \ref{Fig: LinearFairness} in appendix B for an example.
\end{remark}

\subsection{Proof of Theorem 3.4}
\begin{proof}
We first prove the statement on stable points.
Let $\theta^*(\pi)$ denote the optimal policy for the vanilla utility if the proportion of qualified workers is $\pi$. The existence of the stable set $\vec{\theta}_{\text{stab}}(\theta^\text{Maj}_{\text{stab}}, \theta^\text{Min}_{\text{stab}})$ such that $\pi^{\text{Maj}}(\vec{\theta}_{\text{stab}}) \neq \pi^{\text{Min}} (\vec{\theta}_{\text{stab}})$ is equivalent to the existence of multiple intersections of the following functions from $\Theta = [0,1]$ to $ [0,1]$: $$f_1(\theta) =  \theta^{* -1}(\theta),$$
$$f_2(\theta) =  G(w(P(x>\theta|y=1))-P(x>\theta|y=0))$$

Let $Z(\theta) = f_1(\theta) - f_2(\theta)$. Note that $Z(1) = 0 - 0 = 0$. Additionally, consider $Z(\tilde{\theta})$. By the strong convexity assumptions, we can directly calculate that:
\[\theta^{* -1} (\tilde{\theta}) = \frac{\phi(\tilde{\theta}|0)}{\phi(\tilde{\theta}|1) + \phi(\tilde{\theta}|0)} \]
Thus, by assumption, $f_1(\tilde{\theta}) > f_2(\tilde{\theta})$ and since these functions are continuous, there must be an additional $\theta'$ such that $f_1(\theta') = f_2(\theta')$, thus the pair $(\theta', 0)$ will be discriminatory.

Next we prove the statement on optimal policies. We will show that if the vanilla utility $U(\theta, \pi)$ is $\gamma-$ strongly concave and that $w$ is small enough, the performative utility will be concave and thus $\theta^{\text{Maj}}_{\text{opt}} = \theta^{\text{Min}}_{\text{opt}}$ since the optimal policy will be unique.

Consider the decoupled performative utility:
\[U_{\text{perf}}(\theta_1, \theta_2) =\pi(\theta_2) \text{TPR}(\theta_1) + (1-\pi(\theta_2)) \text{FPR}(\theta_2) \] 
The second derivative of $U_{\text{perf}}(\theta)$ is given by the following:
\[U''_{\text{perf}}(\theta) = \frac{\partial^2}{\partial \theta_1^2} U_{\text{perf}}(\theta_1, \theta_2) + \frac{\partial^2}{\partial \theta_2^2} U_{\text{perf}}(\theta_1, \theta_2)\]

We wish to show that $U''_{\text{perf}}(\theta) < -\gamma'$ form some $\gamma' > 0$. Note that by assumption we have the following:
\[ \frac{\partial^2}{\partial \theta_1^2} U_{\text{perf}}(\theta_1, \theta_2) < -\gamma.\]

Letting $\Delta(\theta) = \text{TPR}(\theta) - \text{FPR}(\theta)$
\[\frac{\partial^2}{\partial \theta_2^2} U_{\text{perf}}(\theta_1, \theta_2) = w(w g'(w \Delta(\theta))(\phi(\theta \mid 1) - \phi(\theta \mid 0)) + g(w(\Delta(\theta)))(\phi'(\theta \mid 1) - \phi'(\theta \mid 0)) )\]
Because of the boundedness assumptions $g()$ and $\phi()$ we have that
\[|\frac{\partial^2}{\partial \theta_2^2} U_{\text{perf}}(\theta_1, \theta_2)| < 2wK_1 K_2\]
Thus if $w < \gamma/ 2 K_1 K_2$, the performative firm utility will be strongly concave and thus optimal points will be unique and non-discriminatory.
\end{proof}
\begin{remark}
    Consider the family of markets with $\phi(x\mid 0) = -ax+\frac{a}{2}+1$ and $\phi(x \mid 1) = (n+1)x^{n}$ for some $a$ small and $n$ large. Note that if $p_+ = p_-$ are large enough, this market is strongly concave. Additionally, it is easy to see that there will be a $\tilde{\theta}$ that provides good seperation and additionally that $\phi(\tilde{\theta} \mid 0) \approx 1$ while $\phi(\tilde{\theta} \mid 1) \approx 0$, and thus the lower bound on $w$ in this market can be small as well.
\end{remark}
\section{Discussion on convergence of RRM (the myopic firms process)}
For the purpose of this discussion assume that $X \in [0,1]$. From \citet{perdomo2020Performative}, a myopic firm is sure eventually stabilize if the following conditions are met on the market:
\begin{enumerate}
    \item $\pi(\theta)$ must be $\epsilon-$ Lipschitz continuous, for some $\epsilon$ not large. Note that:
    \[|\pi'(\theta)| \leq w \text{ sup}_{c \in \mathcal{C}}g (c) \text{ sup}_{x \in [0,1]} |\phi(x\mid 1) - \phi(x \mid 0)|\]
    In general, as long as $g()$ and $\phi(|y)$ are bounded functions, then this condition will hold for $\epsilon = w *K$. Additionally, if wages are low, $\epsilon$ will be small.
    \item The firms non-performative utility $U(\theta)$ and the agents aggregate response $\pi(\theta)$ must be sufficiently smooth. This will follow as long as the functions $G(), \phi(\cdot|y)$ are smooth, which is a technical condition and does not impact the dynamics of the market.
    \item The firms vanilla utility $U(\theta) = p_+ \pi \text{TPR}(\theta) - p_- (1-\pi) \text{FPR}(\theta)$ should be $\gamma-$ strongly concave for all $\pi$. Note that:
    \[U''(\theta) = p_+ \pi \phi'(\theta|1) - p_- (1-\pi)\phi'(\theta|0)\]
    For example, if $\phi'(\theta|1) > \gamma/p_+$ and if $\phi'(\theta|0) < -\gamma/p_-$ then $\gamma-$ strong concavity will be guaranteed.
\end{enumerate}
The most crucial ingredient for convergence is the $\epsilon-$ sensitivity of $\pi(\theta)$. In general, if the function
$$\tau(\theta) = \theta^*(\pi) \circ \pi(\theta)$$
$$\theta^*(\pi) = \argmax_{\theta \in [0,1]} U(\theta, \pi)$$
has a derivative with upper bound less then one then convergence of RRM will be ensured. This will occur, for example, if the market is regular enough so that $\theta^*(\pi)$ has a bounded derivative and $w$ is small.

\end{document}